\documentclass[letterpaper]{article} % DO NOT CHANGE THIS
\usepackage{aaai25}  % DO NOT CHANGE THIS

\usepackage{times}  % DO NOT CHANGE THIS
\usepackage{helvet}  % DO NOT CHANGE THIS
\usepackage{courier}  % DO NOT CHANGE THIS
\usepackage[hyphens]{url}  % DO NOT CHANGE THIS
\usepackage{graphicx} % DO NOT CHANGE THIS
\usepackage{natbib}  % DO NOT CHANGE THIS AND DO NOT ADD ANY OPTIONS TO IT
\usepackage{caption} % DO NOT CHANGE THIS AND DO NOT ADD ANY OPTIONS TO IT
\usepackage{algorithm,algorithmicx}
\usepackage[noend]{algpseudocode}
\usepackage{newfloat}
\usepackage{listings}

\usepackage{subfiles}
\usepackage[utf8]{inputenc} % allow utf-8 input
\usepackage[T1]{fontenc}    % use 8-bit T1 fonts
\usepackage{booktabs}       % professional-quality tables
\usepackage{amsfonts}       % blackboard math symbols
\usepackage{nicefrac}       % compact symbols for 1/2, etc.
\usepackage{microtype}      % microtypography
\usepackage{xcolor}         % colors
\usepackage{mathtools}

\usepackage{subcaption}
\usepackage{tabularx}
\usepackage{thmtools, thm-restate}
\usepackage{dsfont}
\usepackage{enumitem}
\usepackage{fontawesome5,xspace}
\usepackage{amsmath,amssymb}
\usepackage{amsthm}
\usepackage{bm}
\usepackage{caption}
\usepackage{subcaption}

\newcommand{\HDETC}{\texttt{HD-ETC}}
\newcommand{\EBHDETC}{\texttt{EBHD-ETC}}
\newcommand{\MLAS}{\texttt{MLAS}}
\newcommand{\KL}{\texttt{kl-UCB}}

\newcommand{\HCLA}{\texttt{HCLA}}
\newcommand{\GHCLA}{\texttt{G-HCLA}}
\newcommand{\GPHCLA}{\texttt{GP-HCLA}}
\newcommand{\HALG}{\texttt{Hungarian}}

\newcommand{\HHMMAB}{\texttt{HMA2B}}
\newcommand{\ghint}{G^{\text{hint}}}
\DeclareMathOperator*{\argmax}{arg\,max}
\DeclareMathOperator*{\argmin}{arg\,min}

\newtheorem{theorem}{Theorem}
\newtheorem{lemma}{Lemma}

\newtheorem{conjecture}{Conjecture}

\newtheorem{definition}{Definition}

\usepackage{array}
\newcolumntype{P}[1]{>{\centering\arraybackslash}p{#1}}

\DeclareMathOperator{\E}{\mathbb{E}}

\newcommand{\1}[1]{\mathds{1}{\{#1\}}}

\newcommand{\abs}[1]{\left\lvert#1\right\rvert}
\usepackage[colorinlistoftodos,prependcaption,textsize=tiny,textwidth=1cm,color=green]{todonotes}

\let\oldComment=\Comment
\renewcommand{\Comment}[1]{\oldComment{\texttt{#1}}}
\algnewcommand{\LeftComment}[1]{\Statex $\triangleright$ \texttt{#1}}
\algnewcommand{\RightComment}[1]{\Statex \leavevmode\hfill$\triangleright$ \texttt{#1}}

\algnewcommand\algorithmicinput{\textbf{Input:}}
\algnewcommand\Input{\item[\algorithmicinput]}%

\algnewcommand\algorithmicoutput{\textbf{Output:}}
\algnewcommand\Output{\item[\algorithmicoutput]}%

\algnewcommand\algorithmicinitial{\textbf{Initialize:}}
\algnewcommand\Initial{\item[\algorithmicinitial]}%

\DeclareCaptionStyle{ruled}{labelfont=normalfont,labelsep=colon,strut=off} % DO NOT CHANGE THIS
\floatstyle{ruled}
\newfloat{listing}{tb}{lst}{}
\floatname{listing}{Listing}
\title{Heterogeneous Multi-Agent Bandits
    with Parsimonious Hints}
\author{
    Amirmahdi Mirfakhar\textsuperscript{\rm 1},
    Xuchuang Wang\textsuperscript{\rm 1},
     Jinhang Zuo\textsuperscript{\rm 2},
    Yair Zick\textsuperscript{\rm 1},
   Mohammad Hajiesmaili\textsuperscript{\rm 1}
}
\affiliations{
    \textsuperscript{\rm 1}University of Massachusetts Amherst,
    \textsuperscript{\rm 2}City University of Hong Kong \\
    smirfakhar@umass.edu, xuchuangw@gmail.com,  jinhang.zuo@cityu.edu.hk,
    yzick@umass.edu,
    hajiesmaili@cs.umass.edu
}

\usepackage{bibentry}
\begin{document}

\maketitle

\begin{abstract}
  We study a hinted heterogeneous multi-agent multi-armed bandits problem (\texttt{HMA2B}), where agents can query low-cost observations (hints) in addition to pulling arms. In this framework, each of the $M$ agents has a unique reward distribution over $K$ arms, and in $T$ rounds, they can observe the reward of the arm they pull only if no other agent pulls that arm.  
The goal is to maximize the total utility by querying the minimal necessary hints without pulling arms, achieving time-independent regret. We study \texttt{HMA2B} in both centralized and decentralized setups. Our main centralized algorithm, \texttt{GP-HCLA}, which is an extension of \texttt{HCLA}, uses a central decision-maker for arm-pulling and hint queries, achieving $O(M^4K)$ regret with $O(MK\log T)$ adaptive hints. In decentralized setups, we propose two algorithms, \texttt{HD-ETC} and \texttt{EBHD-ETC}, that allow agents to choose actions independently through collision-based communication and query hints uniformly until stopping, yielding $O(M^3K^2)$ regret with $O(M^3K\log T)$ hints, where the former requires knowledge of the minimum gap and the latter does not. Finally, we establish lower bounds to prove the optimality of our results and verify them through numerical simulations.

\end{abstract}

% !TeX root = ../main.tex
% \begin{itemize}
%     \item[\todocircle] Can we prove a lower bound for the number of necessary hints? For the lower bound, it depends on the conditions. One cannot formally state that the bounds are optimal. 
%     \item[\todocircle] Other hint models, e.g., give the confidence interval feedback
%     \item[\todocircle] Compare the literature with predictions in the related works. e.g., literature on expert problems. Check multi-agent experts with heterogeneous rewards.
%     \item[\todocircle] Check whether we need to compare with this work: Balkanski, Eric, et al. "Online mechanism design with predictions." arXiv preprint arXiv:2310.02879 (2023).
% \end{itemize}

\section{Introduction}

\begin{table*}[t]
    \centering
    
    \begin{tabularx}{0.80\textwidth}{ccccc}

        \toprule
        Algorithm                       & D/C & Regret                 & Queried Hints                  & Communication           \\
        \toprule
        $\HCLA$ (Algorithm~\ref{routin:hcla})      & C   & $O\left(MK^{2M}\right)$   & $O\left(MK^M \log T\right)$    & N/A                            \\
        \toprule
        $\GHCLA$ (Algorithm~\ref{routin:hcla2})      & C   & $O\left(M^4K\right)$   & $O\left(M^2K \log T\right)$    & N/A                            \\

        \midrule
        $\GPHCLA$ (Algorithm~\ref{routin:hcla1})    & C   & $O\left(M^4K\right)$   & $O\left(MK \log T\right)$      & N/A                            \\
        \midrule
        $\HDETC$ (Algorithm~\ref{routin:hdetc})$^\dagger$     & D   & $O\left(M^3K^2\right)$ & $O\left(M^3K \log (MT)\right)$ & $O\left(\log T\right)$             \\
        \midrule
        $\EBHDETC$ (Algorithm~\ref{routin:ebhdetc}) & D   & $O\left(M^3K^2\right)$ & $O\left(M^3K \log (MT)\right)$ & $O\left(\log T\right)$           \\
        \bottomrule
    \end{tabularx}
    \caption{Regret, Queried Hints, and Communication Bounds for Centralized and Decentralized Algorithms (
    `C' and `D' stand for centralized and decentralized algorithms respectively, 
    $^\dagger$ indicates that \HDETC{} relies on the knowledge of minimum gap)}
    \label{demo-table}
\end{table*}

The multi-agent multi-armed bandit (MA2B) problem~\citep{liu2010distributed,anandkumar2011distributed} is a sequential decision making task consisting of \(K\in\mathbb{N}^+\) arms and \(M\in\mathbb{N}^+\) agents.
In each of the total \(T\in\mathbb{N}^+\) decision rounds,
each agent selects one arm to pull and observes its reward if no other agent pulls the same arm (called no \emph{collision}).
This model has applications in wireless communication~\citep{jouini2009multi,jouini2010upper}, caching~\citep{xu2020collaborative,xu2020decentralized}, and edge computing~\citep{wu2021multi}.
Among various models in MA2B, the heterogeneous multi-agent multi-armed bandit~\citep{bistritz2018distributed,shi2021heterogeneous} is a more realistic variant for these applications where agents have different reward distributions over the arms, e.g., in a wireless communication scenario where agents have different channel qualities due to different geographical locations.
In this heterogeneous MA2B model, the optimal action of all agents is a bipartite matching (between agents and arms) that maximizes the total reward, called the \emph{optimal matching}.
An algorithm's performance is evaluated by \emph{regret}, the difference between the accumulative reward of keeping to choose the optimal matching in all decision rounds and the total reward of the bandit algorithm. A smaller expected regret implies a better algorithm.

 Recently, learning-augmented approaches are emerging, e.g.,~\citet{lykouris2021competitive,bamas2020primal,bhaskara2023bandit}. This stream of research studies how to assist an algorithm with \emph{hints} (a.k.a., predictions) queried from existing ML models, e.g., large language model~\citep{achiam2023gpt}, deep convolutional neural network~\citep{krizhevsky2017imagenet}, and deep reinforcement learning~\citep{franccois2018introduction}.

In this paper, we study the utilization of hint information in heterogeneous multi-agent multi-armed bandits. In addition to receiving feedback from pulling arms, agents can sequentially query hints about the potential rewards of other arms, assisting in their decision-making process.
We call the model \emph{\underline{H}inted \underline{H}eterogeneous \underline{M}ulti-\underline{A}gent \underline{M}ulti-\underline{A}rmed \underline{B}andits} (\texttt{HMA2B}).
Specifically, we consider a simple and accurate hint mechanism where agents can query the reward of an arm without pulling it, with no regret incurred from the queried hint. Despite assuming accurate hints, this model poses challenges, such as balancing hint queries and arm-pullings while accounting for agent heterogeneity and potential future collisions. In addition to minimizing regret, we aim to reduce \emph{hint complexity}, the total number of queried hints, as querying hints, such as via the GPT-4 API~\citep{achiam2023gpt}, can be costly. Efficiently leveraging hints is crucial in scenarios where hint costs are significantly lower than the costs of taking actions. For instance, in labor markets, structured, low-cost interviews provide hints to improve applicant-role matching, reducing the risk of human resource misallocation. Similarly, in radio channel assignments, test signals serve as hints to allocate high-bandwidth channels effectively, preventing delays and disruptions in critical applications like disaster recovery, where drones depend on reliable communication channels.

We study two scenarios of \texttt{HMA2B}: \emph{centralized} and \emph{decentralized} setups.
In the centralized setup, an omniscient decision-maker determines which arm each agent should pull or query hints from, similar to decision-making in hiring processes where the employer has access to the applicants' information to decide which of them to interview and which to hire. In the decentralized setup, agents independently decide their actions through collision-based communications, e.g., in radio channel allocation to stations.

Designing an algorithm that achieves time-independent regret with a linear number of hints in \(T\) is straightforward. However, reducing the queried hints to a sub-linear number in \(T\) is challenging. To tackle this, we first analyze the fundamental limits of hint complexity in the centralized setup and propose \(\GPHCLA\), a fine-tuned algorithm based on the advanced \(\KL\) algorithm~\citep{cappe2013kullback}, which achieves asymptotically optimal hint complexity (Appendix~\ref{App:OofR}). In decentralized setups, the essence of communication lies in the absence of a central decision maker, requiring collision-based signaling~\citep{wang2020optimal}, where making or avoiding collisions encode '1' or '0' information bit. This method introduces inaccuracies from sending decimal statistics in binary and additional regret due to delayed exploration while balancing communication to determine the optimal matching. To address these challenges, we propose \(\HDETC\) and \(\EBHDETC\), which achieve relatively similar bounds on hint complexity and regret. These algorithms use a round-robin hint querying strategy combined with an \emph{Explore-then-Commit}~\citep{garivier2019explore} approach until a stopping condition is met.  Finally, we discuss the optimality of the results in both centralized and decentralized setups in Appendix \ref{App:OofR}.

\subsection{Contributions}
For the centralized setup (Section~\ref{sec:centralized-algorithm}), we propose two algorithms: \(\HCLA\) and \(\GPHCLA\). Both use empirical means to select a matching to pull and \(\KL\) indices~\citep{cappe2013kullback} to identify another matching, querying a hint if the latter has a higher value. Additionally, we analyze an intermediate algorithm, \(\GHCLA\) (Appendix~\ref{section:ghcla}), which operates similarly to \(\HCLA\) but differs from \(\GPHCLA\) in how it selects the matching to hint after deciding to query. As summarized in Table~\ref{demo-table}, both \(\GPHCLA\) and \(\GHCLA\)—extensions of \(\HCLA\)—achieve time-independent regret with an asymptotically optimal number of hints. We further prove that the upper bound on the hint complexity for \(\GPHCLA\) is tight, with both \(\GPHCLA\) and \(\GHCLA\) matching the established lower bounds. In the decentralized setup (Section~\ref{sec:decentralized-algorithm}), we introduce two algorithms: \(\HDETC\) and \(\EBHDETC\). Both divide the time horizon into three phases: \emph{exploration}, \emph{communication}, and \emph{exploitation}, with a key difference in how they transition to the exploitation phase. In the exploitation phase, no further communication, exploration, or hint querying occurs, and the two algorithms handle this transition differently. In \(\HDETC\), agents know the minimum gap—the smallest utility difference between the optimal and other matchings—and the time horizon \(T\), allowing them to switch to exploitation at a fixed time step \(T_0\). Conversely, \(\EBHDETC\) does not require this knowledge, using an edge elimination strategy to determine the transition point, which makes it a random variable. This results in slightly higher hint queries and regret compared to \(\HDETC\). We provide regret bounds for both algorithms that align with the lower bounds, accounting for uncertainties due to delayed communication.

\subsection{Related Works}

\paragraph{Heterogeneous MMAB (HMMAB)}
HMMAB is one of the standard models in multi-player multi-armed bandits with collision literature; to name a few, \citet{rosenski2016multi,boursier2019sic,mehrabian2020practical,bistritz2018distributed,shi2021heterogeneous}.
Among them, \citet{bistritz2018distributed} was the first to study the HMMAB, where they proposed a decentralized algorithm with $O(\log^2 T)$ regret.
Later on, the regret bound of this model was improved to \(O(M^3K\log T)\) by~\citet{mehrabian2020practical} and further to \(O(M^2K\log T)\) by~\citet{shi2021heterogeneous} that is the state-of-the-art result.
We are the first to introduce the hint mechanism to HMMAB.

\paragraph{Bandits with Hints}
Learning algorithms with hints (or predictions) are part of the emerging literature on learning-augmented methods, as seen in works like~\citep{lykouris2021competitive,purohit2018improving,mitzenmacher2022algorithms}, etc. The hint mechanism was initially explored in the basic stochastic multi-armed bandits model by~\citet{yun2018multi}. Later, \citet{lindstaahl2020predictive} examined a more realistic hint mechanism, which includes failure noise, for the same model. Additionally, \citet{bhaskara2023bandit} investigated the impact of hints in adversarial bandits. We are the first to study the hint mechanism in a multi-agent scenario.

% \mo{make 1.3 a new section.}
\section{Hinted Heterogeneous Multi-Agent Multi-Armed Bandits}
\paragraph{Basic model} A Hinted Heterogeneous Multi-agent Multi-Armed Bandit (\texttt{HMA2B}) model consists of a set of $K$ \emph{arms}  \(\mathcal{K}\)  and a set of $M$ \emph{agents} \(\mathcal{M}\), such that \(M < K\).
Agents have heterogeneous \emph{rewards} for arms. That is, for each agent \(m\in \mathcal{M}\), each arm \(k\in\mathcal{K}\) is associated with a Bernoulli reward random variable \(X_{m,k}\) with mean \(\mu_{m,k} \coloneqq \mathbb{E}[X_{m,k}]\).
% This implies that several agents are involved in the same bandit problem, with a real-valued parameter $\mu_{m,k} \in (0,1)$ for each $m \in \mathcal{M}$ and $k \in \mathcal{K}$. 
The heterogeneous reward means are represented by a matrix $\bm{\mu} \in [0,1]^{M\times K}$, where each of its rows is denoted by $\bm{\mu}_m = (\mu_{m,k})_{k\in\mathcal{K}}\in [0,1]^{K}$.
% When player \(m\in\mathcal{M}\) pulls arm \(k\in\mathcal{K}\), it receives a reward $r_{m}$ sampled from the Bernoulli distribution \(\mathcal{D}_{m,k}\) with \(\mu_{m,k}\) as its mean. 
% where the reward value lies in the interval \([0,1]\)\todo{as KL-UCB is applied, please assume Bernoulli reward directly}. 

% We also will use another notation \(G_m^{-1}\), which returns the only arm that agent $m$ is connected to, $k_m$, in graph \(G\)\todo{\(G_m^{-1}\) and \(k_m\) are the same notation, why not use some notation like \(k_m^{G}\) instead?}. 
\paragraph{Reward feedback} Suppose that \(T\in\mathbb{N}^+\) denotes the total number of decision rounds. At each \emph{time step} \(t\in \{1,2,\dots, T\}\), every agent \(m\) chooses an arm \(k_{m}(t)\) to pull. The arms requested by the agents construct a bipartite graph characterized by \(M\) nodes (agents) on one side and \(K\) nodes (arms) on the other comprising \(M\) edges, ensuring that each node on the agent side is connected to exactly one arm.
Let us define \(\mathcal{G}\) as the set of all such graphs. Denote \(G(t) \coloneqq (m, k_m(t))_{m\in\mathcal{M}}\) as the bipartite graph representing the arm pulling graph of the agents at time step \(t\).
We consider the \emph{collision} setting~\citep{boursier2019sic,shi2021heterogeneous}: that is, if there exist other agents pulling the arm \(k_m(t)\) at time step \(t\), then agent \(m\) gets a reward of zero; otherwise, agent \(m\) gets a reward \(X_{m,k_m(t)}(t)\) sampled from the reward distribution of arm \(k_m(t)\), or formally, \(r_m(t) \coloneqq X_{m,k_m(t)}(t)\1{\forall m' \neq m: k_{m'}(t) \neq k_m(t)}\). This induces the optimal action to be a matching.

Given a matching $G\in\mathcal{G}$ and a reward mean matrix ${\bm{\mu} \in [0,1]^{M\times K}}$, we define the expected utility as
% {\footnotesize
        \begin{align*}
            U(G;\bm{\mu}) & \coloneqq \mathbb{E}\left[\sum_{m \in \mathcal{M}}r_{m}\right] \\ &= \sum_{m\in\mathcal{M}} \mu_{m,k_m^G} \1{\forall m' \neq m: k_{m'}^G \neq k_m^G},
        \end{align*}
    % }
where \(k_m^G\) denotes the matched arm of agent \(m\) under matching \(G\).
We denote the matching with the highest utility as the optimal matching \(G^* \coloneqq \max_{G \in \mathcal{G}} U(G;\bm{\mu}).\) We assume that $G^*$ is \emph{unique}, i.e., there does not exist any $G\neq G^*$ in $\mathcal{G}$ such that $ U(G;\bm{\mu}) = U(G^*;\bm{\mu})$.

\paragraph{Hint mechanism} At each time slot \(t\), besides the pulled arm \(k_m(t)\), agent \(m\) can query another arm \(k_m^{\text{hint}}(t)\) and observe the arm's reward realization \(X_{m,k_m^{\text{hint}}(t)}(t)\) without regret cost. The hint graph then is denoted by $G^{\text{hint}}(t)$ and $k^{G^{\text{hint}}(t)}_m$ is the arm agent $m$ queried a hint for in it. These hint observations do not impact the accumulative reward and regret, and the agent can decide whether to query for a hint, denoted by the indicator function \(\ell^\pi_m(t) \coloneqq \1{\text{agent }m\text{ query a hint at }t\text{ under policy }\pi}\).
We denote \(L^\pi(T) \coloneqq \E\left[\sum_{m\in\mathcal{M}}\sum_{t=1}^T \ell^\pi_m(t)\right]\) as the total number of times of agents querying hints, and we want to design a learning policy $\pi$ minimizes the \(L^\pi(T)\) while maintaining low regret.
% \mo{Is hint collision free? clarify that.}

\paragraph{Regret.} We aim to find a policy $\pi$ that maximizes the cumulative reward of all agents by determining $G(t)$ at each round in the \(T\) rounds.
To evaluate the performance of $\pi$, we define the \emph{regret} of a policy as the difference between the total reward of all agents under the optimal matching \(G^*\) in all decision rounds and the total reward of all agents following the policy \(\pi\), as follows,
\begin{align}
    R^{\pi}(T) \coloneqq \sum_{t=1}^{T} U(G^*;\bm{\mu}) - \mathbb{E}\left[ U(G(t);\bm \mu )\right],
\end{align}
where the expectation is taken over the randomness of the policy \(\pi\). Last, we define the important parameter, the minimum gap, which is crucial and appears in our regret analysis. The \emph{minimum gap} here represents the minimum difference between the utility of any matching $G$ and $G^*$, i.e., $\Delta^{\text{match}}_{\min} \coloneqq \min_{G \neq G^* \in \mathcal{G}} U(G^*;\bm{\mu}) - U(G;\bm{\mu})$.

% According to the definition of regret, our proposed algorithms will consider each $G \in \mathcal{G}$ as a ``super arm''; thus, t
\paragraph{Main goal and motivating examples}
Our goal is to design learning policies that use hints—one per agent at a time—to reduce the large regret bounds established in previous works \citep{shi2021heterogeneous,mehrabian2020practical,wang2020optimal,boursier2019sic} to a preferably time-independent regret, while minimizing the number of hints queried. We assume that hints are sampled from the same distributions as the rewards from pulling arms. Our algorithms query these hints strategically, only when exploring a sub-optimal matching is necessary before committing to the optimal one. This approach minimizes the costs of direct exploration and improves performance by separating the exploration of sub-optimal matchings from the exploitation of the optimal one.

% This method also allows for the consideration of more complex hint models.

In practical scenarios, hints are typically much cheaper than direct actions. For instance, in labor markets, a low-cost interview process can provide valuable insights into candidate suitability without the high costs of hiring mistakes. Similarly, in communication networks, using test signals to estimate bandwidth needs can prevent wasting high-quality channels on low-demand stations. These examples demonstrate how the hint-based approach in \texttt{HMA2B} can improve decision-making across various applications.

\section{Algorithms for Centralized Hinted Heterogeneous Multi-Armed Bandits}\label{sec:centralized-algorithm}
In the \emph{Centralized Hinted Heterogeneous Multi-Armed Bandit} ($\texttt{C}\_\HHMMAB$) setup, we consider an \emph{omniscient} decision maker who selects both the matching and the hint graph at each round. The agents then follow the decision maker's instructions to pull arms and query hints. We propose two learning policies for this setup: the \emph{Hinted Centralized Learning Algorithm} ($\HCLA$) and the \emph{Generalized Projection-based Hinted Centralized Learning Algorithm} ($\GPHCLA$).

Under both policies, the decision maker treats each matching \(G \in \mathcal{G}\) as a super arm for hint inquiries. However, the handling of observations differs between the two: in $\HCLA$, observations are maintained for each matching, while in $\GPHCLA$, they are treated at the edge level. This distinction allows us to reduce the potentially exponential regret relative to the size of \(\mathcal{G}\) to a polynomial regret upper bound in the number of edges, \(MK\).

We first introduce the statistics maintained by agents in $\HCLA$ and $\GPHCLA$, aiding the central decision maker in deciding when and how to query hints. Next, we describe $\HCLA$ as a baseline for designing $\GPHCLA$, our main algorithm. We also present an intermediate algorithm, $\GHCLA$, as a direct extension of $\HCLA$. Finally, we detail $\GPHCLA$, which requests hints more efficiently than $\GHCLA$. $\GHCLA$ is further discussed in Appendix~\ref{section:ghcla}.

\subsection{Preliminaries}
Beyond the generic empirical means matrix $\bm{\hat{\mu}}$, the decision maker employs $\KL$ indices $\bm{d}$ \citep{cappe2013kullback} as upper confidence bounds for $\bm{\mu}$ in the $\texttt{C}\_\HHMMAB$ setup to determine when to query a hint. These indices are defined as:  
{\footnotesize  
\begin{align}
\hat{\mu}_{G}(t)    & \coloneqq \frac{\sum_{t'=1}^t \1{G(t') = G}U(G(t');\bm{r}(t'))}{ N^\pi_{G}(t)},  \\                                                                                              
\hat{\mu}_{m,k}(t)  &\coloneqq \frac{\sum_{t'=1}^t \1{(m,k)\in G(t')}r_m(t')}{ N^\pi_{m,k}(t)},                                                                                                     \\
    d_{G}(t)    & \coloneqq \sup \left\{q \geq 0: N^\pi_{G}(t)\ \mathrm{kl}\left(\hat{\mu}_{G}(t), q\right) \leq f(t) \right\}, 
    \label{ineq:KLsuper}                                                                                                                               \\
    d_{m, k}(t) & \coloneqq \sup \left\{q \geq 0: N^\pi_{m,k}(t)\ \mathrm{kl}\left(\hat{\mu}_{m,k}(t), q\right) \leq  f(t)\right\},
    \label{ineq:KLedge}
\end{align}  
}

\noindent for any matching \(G \in \mathcal{G}\) and edge \((m,k) \in \mathcal{M} \times \mathcal{K}\), respectively, where $\mathrm{kl}$ is the Kullback-Leibler divergence, and \(f(t) = \log t + 4 \log \log t\). Here, \(N^\pi_{G}(t)\) and \(N^\pi_{m,k}(t)\) represent the number of times matching \(G\) or edge \((m,k)\) has been pulled or hinted.

Before detailing the algorithm, we define a fixed set \(\mathcal{R} \coloneqq \{R_1, \dots, R_K\}\) of \(K\) pairwise edge-disjoint matchings that cover all edges \((m,k) \in \mathcal{M} \times \mathcal{K}\), referred to as \emph{covering matchings}. For uniquely labeled agents and arms in \([M]\) and \([K]\), \(R_i \in \mathcal{R}\) is the matching where agent \(m\) is paired with arm \((m + i - 1) \bmod K\), as shown in Figure \ref{fig:rr} for \(M=3\) and \(K=4\). By pulling or hinting each covering matching in \(\mathcal{R}\) at least once, agents can observe all \(G \in \mathcal{G}\) at least once. This set serves as a \emph{hint pool}, from which all hint graphs \(G^{\text{hint}}\) will be selected.

% The concept of having pairwise edge-disjoint matchings covering all edges as a hint pool allows us to reduce the number of queried hints from the exponentially large set \(\mathcal{G}\) to the considerably smaller set \(\mathcal{R}\), which contains \(K\) members.

% These indices implicitly give the decision maker an upper confidence bound for the expected utility of each matching or edge, offering a sense of whether the current estimations are accurate enough based on the number of observations for either a matching or an edge. The decision maker then provides agents with hints if the upper confidence bound of some matching $G' \in \mathcal{G}$, referred to as $d_{G'}$, exceeds the utility of the empirically best matching found so far.

% We construct \(\mathcal{R}\) by including \(R_1\) and \(R_i\) for all \(i \in \{2, 3, \cdots, K\}\), where \(R_i\) is obtained by applying \(i-1\) consecutive shifts to \(R_1\). This results in edges of the form \((m, (m+i)\%K)\), as shown in Figure \ref{fig:rr} for \(M=3\) and \(K=4\).
% This structured approach to defining the hint pool simplifies the selection of hint graphs, thereby reducing the computational complexity associated with hint inquiries in the \(\text{\HHMMAB}\) model.

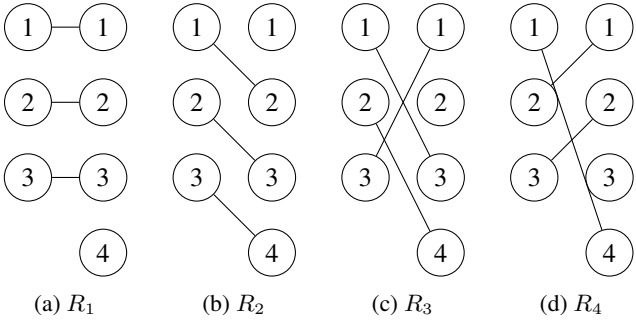
\begin{figure}[!t]

    \centering
    \begin{subfigure}[b]{0.2\linewidth}%
        \centering

        \begin{tikzpicture}[main/.style = {draw, circle},node distance=1cm and 0.5cm]
            \node[main] (1) {1};
            \node[main] (2)[below of=1] {2};
            \node[main] (3)[below of=2] {3};

            \node[main] (5)[right of=1] {1};
            \node[main] (6)[below of=5] {2};
            \node[main] (7)[below of=6] {3};
            \node[main] (8)[below of=7] {4};

            \draw (1) -- (5);
            \draw (2) -- (6);
            \draw (3) -- (7);

        \end{tikzpicture}
        \caption{$R_1$}
    \end{subfigure}
    \hfill
    \begin{subfigure}[b]{0.2\linewidth}
        \centering
        \begin{tikzpicture}[main/.style = {draw, circle},node distance=1cm and 2cm]
            \node[main] (1) {1};
            \node[main] (2)[below of=1] {2};
            \node[main] (3)[below of=2] {3};
            \node[main] (5)[right of=1] {1};
            \node[main] (6)[below of=5] {2};
            \node[main] (7)[below of=6] {3};
            \node[main] (8)[below of=7] {4};

            \draw (1) -- (6);
            \draw (2) -- (7);
            \draw (3) -- (8);

        \end{tikzpicture}
        \caption{$R_2$}
    \end{subfigure}
    \hfill
    \begin{subfigure}[b]{0.2\linewidth}
        \centering
        \begin{tikzpicture}[main/.style = {draw, circle},node distance=1cm and 2cm]
            \node[main] (1) {1};
            \node[main] (2)[below of=1] {2};
            \node[main] (3)[below of=2] {3};

            \node[main] (5)[right of=1] {1};
            \node[main] (6)[below of=5] {2};
            \node[main] (7)[below of=6] {3};
            \node[main] (8)[below of=7] {4};

            \draw (1) -- (7);
            \draw (2) -- (8);
            \draw (3) -- (5);

        \end{tikzpicture}
        \caption{$R_3$}
    \end{subfigure}
    \hfill
    \begin{subfigure}[b]{0.2\linewidth}
        \centering
        \begin{tikzpicture}[main/.style = {draw, circle},node distance=1cm and 2cm]
            \node[main] (1) {1};
            \node[main] (2)[below of=1] {2};
            \node[main] (3)[below of=2] {3};

            \node[main] (5)[right of=1] {1};
            \node[main] (6)[below of=5] {2};
            \node[main] (7)[below of=6] {3};
            \node[main] (8)[below of=7] {4};

            \draw (1) -- (8);
            \draw (2) -- (5);
            \draw (3) -- (6);

        \end{tikzpicture}
        \caption{$R_4$}
    \end{subfigure}

    \caption{Set of covering matchings \(\mathcal{R}\) for $M=3$ and $K=4$: $R_1$, $R_2$, $R_3$ and $R_4$ are depicted in (a), (b), (c) and (d).}
    \label{fig:rr}
\end{figure}

\subsection{Warm-up: The $\HCLA$ Algorithm}

As noted earlier, the $\HCLA$ algorithm treats each \(G \in \mathcal{G}\) as a super arm and maintains separate statistics: empirical mean \(\hat{\mu}_{G}(t)\), \(\KL\) index \(d_G(t)\), and counters \(N^\HCLA_G(t)\). At each time step \(t\), the central decision maker selects a matching \(G(t)\) with the maximum empirical mean \(\hat{\mu}_G(t)\) and another matching \(G'(t)\) with the maximum \(d_G(t)\) (Lines~\ref{line:empirical-G}--\ref{line:hint-G}). If \(d_{G'(t)}(t) > \hat{\mu}_{G(t)}(t)\), the decision maker chooses \(\ghint(t)\) as either \(G'(t)\) or a uniformly at random chosen matching \(\ghint_2(t)\), each with probability \(\nicefrac{1}{2}\). It then queries a hint from \(\ghint(t)\) and updates \(\hat{\mu}_{\ghint(t)}(t+1)\) based on the hint observation (Lines~\ref{l:line:klh}--\ref{line:end-query}). Finally, the decision maker pulls \(G(t)\), updates \(\hat{\mu}_{G(t)}(t+1)\) with the reward observation, and recalculates \(d_{G}(t+1)\) for all \(G \in \mathcal{G}\) (Lines~\ref{line:pull-G}--\ref{line:update-mu-kl-ucb}). The detailed pseudocode of the $\HCLA$ algorithm is provided in Algorithm~\ref{routin:hcla}.

% and the agents pull the arms according to \(G(t)\).
% Every time there exists a matching $G' \neq G(t)$ such that $d_{G'}(t) > \hat{\mu}_{G(t)}(t)$, the central decision maker then sets one such matching $G'$ with the maximum $d_{G'}(t)$ to $\ghint(t)$, and agents ask for hints according to $\ghint(t)$ . The algorithm selects both $G(t)$ and $\ghint(t)$ by taking the maximum over the exponentially-sized set \(\mathcal{G}\).

\begin{algorithm}[tp]
    \caption{ Hinted Centralized Learning Algorithm ($\HCLA$) }\label{routin:hcla}
    
    \begin{algorithmic}[1]
    \Input  agent set \(\mathcal{M}\), arm set \(\mathcal{K}\), number of agents $M$, matching set $\mathcal{G}$, time horizon $T$
        \State \textbf{Initialization:} $t\gets 0$, $\hat{\mu}_{G}(t)\leftarrow 0$, $d_{G}(t)\leftarrow 0$, $N_{G}(t) = 0$ for each matching $G \in \mathcal{G}$
        \For{$ t \in  [T]$}
        \State $G(t) \leftarrow \argmax_{G \in \mathcal{G}} \hat{\mu}_{G}(t)$ \label{line:empirical-G}
        \State $G'(t) \leftarrow \argmax_{G \in \mathcal{G}} d_{G}(t)$ \label{line:hint-G}
        \If{$d_{G'(t)}(t) > \hat{\mu}_{G(t)}(t)$} \label{l:line:klh}
        \State $G^{\textit{hint}}_1(t)\gets G'(t)$
        \State $\ghint_2(t) \gets$  pick a matching out of $\mathcal{G}$ uniformly at random \label{line:hint-choose-2}
        \State $\ghint(t) \gets
            \begin{cases}
                \ghint_1(t), & \text{w.p. } \frac{1}{2} \\
                \ghint_2(t), & \text{w.p. } \frac{1}{2}
            \end{cases}$ \label{line:1-2-hint-choose}

        \State Each agent $m$ asks for a hint from $k^{\ghint(t)}_m$
        \State Update $\hat{\mu}_{G^{\textit{hint}}(t)}(t+1)$ according to the observation of \(\ghint(t)\)
        \EndIf\label{line:end-query}
        \State Each agent $m$ pulls $k^{G(t)}_m$ \label{line:pull-G}
        \State Update $\hat{\mu}_{G(t)}(t+1)$ and  according to the reward observation of \(G(t)\)
        \State Update $d_{G}(t+1)$ for all $G\in\mathcal{G}$\label{line:update-mu-kl-ucb}
        \EndFor
    \end{algorithmic}
\end{algorithm}

Next, we present the upper bounds for the time-independent regret and the number of queried hints $L^\HCLA(T)$ for the $\HCLA$ algorithm in Theorem \ref{thm:matirphi}. The detailed proof is presented in Appendix~\ref{proof:hcla}.
\begin{restatable}{theorem}{hclap}\label{thm:matirphi}
    For $0< \delta < \frac{\Delta^\text{match}_{\min}}{2}$ and policy $\pi = \HCLA$, the policy $\pi$ has
    \begin{enumerate}
        \item time-independent regret  $R^\pi(T) \in O\left(MK^{2M}\right)$,
        \item hint complexity $L^\pi(T) \in O\left(\frac{MK^M \log T }{\Delta^{\mathrm{kl}}}\right),$
    \end{enumerate}
    where $\Delta^{\mathrm{kl}} = \mathrm{kl}(U(G^*;\bm{\mu}) - \Delta^\text{match}_{\min} + \delta, U(G^*;\bm{\mu}) - \delta)$.
\end{restatable}

The regret of $\HCLA$ is time-independent, but the exponential constants in its regret and hint upper bounds are unsatisfactory. To address this, we propose a new algorithm called $\GPHCLA$, which provides a more refined analysis while maintaining the same hint inquiry and arm-pulling approach but using observations differently.

\subsection{The $\GPHCLA$ Algorithm}

We present $\GPHCLA$ in Algorithm~\ref{routin:hcla1}. The $\GPHCLA$ algorithm follows steps similar to $\HCLA$ to identify \(G(t)\) and \(G'(t)\). However, unlike $\HCLA$, the central decision maker maintains statistics \(\hat{\mu}_{m,k}(t)\) and \(d_{m,k}(t)\) for each edge \((m,k) \in \mathcal{M} \times \mathcal{K}\). It then defines \(d_G(t) \coloneqq \sum_{(m,k)\in G}d_{m,k}(t)\), with a slight abuse of notation, enabling the use of the \(\HALG\) algorithm \citep{kuhn1955hungarian}, which finds the matching with maximum additive utility in a weighted bipartite graph. Accordingly, $\GPHCLA$ utilizes \(\HALG\) to compute \(G(t)\) and \(G'(t)\), where the weights of the edges \((m,k)\) are \(\hat{\mu}_{m,k}(t)\) and \(d_{m,k}(t)\), respectively (Lines~\ref{line:hungarian-G1}--\ref{line:hungarian-hint-G1}). The decision maker employs a distinctly different approach from $\HCLA$ for selecting \(\ghint(t)\) and updating the statistics after each observation, whether from pulling an arm or querying a hint. As in $\HCLA$, the algorithm queries for a hint if \(U(G'(t);\bm{d}(t)) > U(G(t);\bm{\hat{\mu}}(t))\) (Line~\ref{line:hint-query1}). However, instead of querying a hint directly from \(G'(t)\), $\GPHCLA$ \emph{projects} \(G'(t)\) onto a matching in \(\mathcal{R}\), a set of pairwise edge-disjoint covering matchings. By projection, we mean mapping \(G'(t) \in \mathcal{G}\) to a matching in \(\mathcal{R}\), which contains \(K\) covering matchings and is exponentially smaller. During this step, the algorithm selects the matching \(\ghint_1(t)\) from \(\mathcal{R}\) that contains the edge \((m,k) \in G'(t)\) with the fewest \(N^{\GPHCLA}_{m,k}(t)\), and a second matching \(\ghint_2(t)\), chosen uniformly at random from \(\mathcal{R}\). The hint graph \(\ghint(t)\) is then set to either \(\ghint_1(t)\) or \(\ghint_2(t)\), each with probability \(\nicefrac{1}{2}\) (Lines~\ref{line:hint-choose-11}--\ref{line:end-hint-query1}).

\begin{algorithm}[tp]
    \caption{ Generalized Projection-based Hinted Centralized Learning Algorithm ($\GPHCLA$) }\label{routin:hcla1}
    \begin{algorithmic}[1]
        % \Input set of agents $\mathcal{M}$, set of arms  $\mathcal{K}$, time horizon $T$
        \Input agent set $\mathcal{M}$, arm set \(\mathcal{K}\), time horizon \(T\), 
        \State \textbf{Initialization:} $t\gets 0$,  $\hat{\bm \mu}_{m}(t)\leftarrow \bm 0$,  $\bm{d}_{m}(t)\leftarrow \bm 0$,  $\bm{N}^{\GPHCLA}_{m}(t)\gets \bm 0$ for each agent $m \in \mathcal{M}$

        \For{$t \in T$}

        \State $G(t) \leftarrow \HALG\left(\bm{\hat{\mu}}(t)\right)$
        \label{line:hungarian-G1}

        \State $G'(t) \leftarrow \HALG\left(\bm{d}(t)\right)$\label{line:hungarian-hint-G1}
        \If{$U(G'(t);\bm{d}(t)) > U(G(t);\bm{\hat{\mu}}(t))$}\label{line:hint-query1}
        \State $( m,k ) \leftarrow \argmin_{(m',k') \in G'(t)}N^{\GPHCLA}_{m',k'}\left(t\right)$\label{line:hint-choose-11}
        \State $\ghint_1(t) \leftarrow \{R\in \mathcal{R}: (m,k) \in R\}$ 
        \State $\ghint_2(t) \gets$  pick a matching out of $\mathcal{R}$ uniformly at random
        \State $\ghint(t) \gets
            \begin{cases}
                \ghint_1(t), & \text{w.p. } \frac{1}{2} \\
                \ghint_2(t), & \text{w.p. } \frac{1}{2}
            \end{cases}$
        \State Each agent $m$ asks for a hint from $k^{\ghint(t)}_m$
        \label{line:end-hint-query1}
        % \State Update
        \EndIf
        \State Each agent $m$ pulls $k^{G(t)}_m$
        \State Update $\bm{\hat{\mu}}_m(t+1)$, \(\bm N^{\GPHCLA}_{m}\left(t+1\right)\), and $\bm {d}_m(t+1)$ for each agent $m$ according to new observations

        \EndFor
    \end{algorithmic}
\end{algorithm}

% As previously mentioned, $\GPHCLA$ keeps different stats at the edge level rather than at the matching level, as maintained by $\HCLA$. This approach enables the decision maker to run $\HALG$ algorithm, which has polynomial time complexity, rather than taking a maximum over an exponentially large set $\mathcal{G}$, which would result in a very slow algorithm. 

In Theorem \ref{thm:c_whhmmab}, we provide the bound for the regret and the asymptotically optimal bound for the number of hints. The detailed proof is presented in Appendix~\ref{proof:gphcla}.

\begin{restatable}{theorem}{gphclap}\label{thm:c_whhmmab}
    For $0< \delta < \frac{\Delta^\text{match}_{\min}}{2}$ and policy $\pi = \GPHCLA$, the policy $\pi$ has
    \begin{enumerate}
        \item time-independent regret $R^\pi(T) \in O\left(M^4K\right)$ regret,
        \item hint complexity $L^{\pi}(T) \in O\left(\frac{MK \log T}{\Delta^{\mathrm{kl}}}\right)$,
    \end{enumerate}
where $\Delta^{\mathrm{kl}} = \mathrm{kl}(U(G^*;\bm{\mu}) - \Delta^\text{match}_{\min} + \delta, U(G^*;\bm{\mu}) - \delta)$.
\end{restatable}

Theorem \ref{thm:c_whhmmab} highlights the impact of maintaining edge-level statistics in $\GPHCLA$, reducing the exponential time-independent regret bound to a polynomial. It also shows that projection in hint inquiries minimizes hints, achieving asymptotic optimality (matching the lower bound given by Theorem~\ref{thm:hint-lower-bound} in the Appendix). We study $\GHCLA$, an extension of $\HCLA$, which updates statistics like $\GPHCLA$ but skips projection, using $\ghint(t)$ as in $\HCLA$. Theorem \ref{lemma:hcla2} (Appendix) shows $\GHCLA$ can have up to \(M\)-times higher hint complexity than $\GPHCLA$, highlighting the importance of projection. Experiments (Appendix \ref{App:OofR}, Figure \ref{fig1:chints}) confirm $\GPHCLA$ outperforms $\GHCLA$ on small problem instances. The exact tightness of this gap remains open due to the complexity of the $\KL$ index.

\section{Algorithms for Decentralized Hinted Heterogeneous Multi-Armed Bandits} \label{routine:decgame}\label{sec:decentralized-algorithm}

We study the \emph{Decentralized Hinted Heterogeneous Multi-Armed Bandits} (\(\texttt{D}\_\HHMMAB\)s), where no central decision maker coordinates agents to avoid collisions while learning the optimal matching \(G^*\). Theorem \ref{obs:msg} demonstrates that sub-linear regret is unattainable in a decentralized setup without agents sharing statistics, making communication essential in \(\texttt{D}\_\HHMMAB\)s. To enable communication, agents intentionally collide to exchange statistics like \(\bm{\hat{\mu}}\)s, while non-colliding agents continue pulling their assigned arms \(k^{G}_m\) from the matching \(G \in \mathcal{G}\) without interference~\citep{shi2021heterogeneous,wang2020optimal}. Communication order is determined by unique agent ranks, as discussed below.

\begin{restatable}[Necessity of Communication]{theorem}{obsp}\label{obs:msg}
    No decentralized learning algorithm can achieve sub-linear instance-independent regret in $\HHMMAB$s without communication.
\end{restatable}

Building on Theorem~\ref{obs:msg}, we propose a cooperative learning framework for the \emph{Hinted Decentralized Explore-then-Commit} ($\HDETC$) and \emph{Elimination-Based Hinted Decentralized Explore-then-Commit} ($\EBHDETC$) algorithms. These divide time into \emph{Initialization}, \emph{Exploration}, and \emph{Communication} phases, where agents request hints in a round-robin manner until meeting a \emph{stopping condition}, after which they transition to the \emph{Exploitation} phase with no further hints or communication.

\subsection{Decentralized Learning Framework}\label{routine:framework}
We first outline the common framework for the $\HDETC$ and $\EBHDETC$ algorithms. Both divide the \(T\) decision-making rounds into alternating exploration and communication phases. A counter \(\rho\) tracks exploration epochs, and \(N_{m,k}^\rho\) records the number of times agent \(m\) has pulled or hinted at arm \(k\) by the start of epoch \(\rho\). The decentralized learning framework for $\HHMMAB$ consists of four phases:

\begin{description}[leftmargin=0pt]

    \item[Initialization phase:] Assigning unique ranks among the agents. The detailed rank assignment procedure and analysis follows~\citet{wang2020optimal} (detailed in Appendix \ref{routine:rank}).
    
    \item[Exploration phase:] Agents use the gathered statistics to identify the best matching \(G^\rho\) at the start of each epoch \(\rho\) using $\HALG$ algorithm. They then commit to their corresponding arm \(k^{G^\rho}_m\) for \(K\) rounds until the epoch ends. At the end of epoch \(\rho\), agents signal the communication phase by creating collisions on arms pulled by other agents.

    \item[Communication phase:] Before each exploration epoch \(\rho\), agents transmit their statistics \(\bm{\hat{\mu}}\) to others, denoted as \(\bm{\hat{\mu}}^\rho\). This communication is realized via the intentional collision signals, where a collision represents a '1' and its absence a '0' information bits \cite{boursier2019sic}, with agents relying on their unique ranks to identify senders and receivers. Since \(\bm{\hat{\mu}}^\rho\) often contains decimal values, agents transmit a quantized version, \(\bm{\tilde{\mu}}^\rho\), optimized for binary communication at the cost of minor information loss. To further reduce communication length and minimize information loss, agents employ the \emph{Differential} communication~\cite{shi2021heterogeneous}, sending only the differences \(\bm{\tilde{\delta}}^\rho = \bm{\tilde{\mu}}^\rho - \bm{\tilde{\mu}}^{\rho-1}\) at the start of epoch \(\rho\). This method reduces communication-induced regret through $t^\rho_\text{com} \in O(M^2K)$ communication rounds. It enables agents to synchronize actions and exchange critical information efficiently via the $\operatorname{Send2All(\bm{\tilde{\delta}}^\rho)}$ routine, detailed in Appendix~\ref{routine:com}.

    % Hint inquiry is a key part of the algorithms we, however, agents can query a hint at every time step regardless of whether it is the exploration or communication phase.
    \item[Exploitation phase:]
   Agents stop communicating, exploring, and querying for hints after a specific time \(T^\pi_0\), which depends on the policy $\pi$ being used. After that, they agree on a matching \(G'^*\) and commit to it for the rest of the time.

\end{description}

Unlike~\citet{shi2021heterogeneous} employing exponentially increasing exploration epoch lengths summing to \( O(\log T) \) epochs, our approach simplifies this by assigning each epoch the same length \( K \), resulting in potentially \( O\left(\frac{T}{K}\right) \) epochs. However, with stop conditions, our algorithms reduce the number of epochs and transition to the \textbf{exploitation phase} while maintaining \( O(\log T) \) exploration epochs.

\paragraph{Hint inquiry mechanism}
$\HDETC$ and $\EBHDETC$ employ a round-robin approach for querying hints, setting \( G^\text{hint}(t) = R_{(t\%K)+1} \), and follow an Explore-then-Commit exploration style. By evenly distributing hint queries over \( K \) rounds, this method reduces communication costs and prevents time-dependent regret. In comparison, the decentralized $\HCLA$ queries hints on demand, requiring constant communication and potentially incurring linear regret.

\paragraph{Regret decomposition}
% \hfill \break
We decompose the regret as follows to analyze its components separately:  
\[
R^{\pi}(T) = R^{\pi_{\text{rank}}}(T) + R^{\pi_{\text{exp}}}(T) + R^{\pi_{\text{com}}}(T),
\]  
where \( R^{\pi_{\text{rank}}}(T) \), \( R^{\pi_{\text{exp}}}(T) \), and \( R^{\pi_{\text{com}}}(T) \) represent the regret due to `rank assignment,' `exploration,' and `communication,' respectively, under policy \(\pi\).

Under this framework, we introduce the $\HDETC$ and $\EBHDETC$ algorithms in the following sections.

\subsection{Warm-Up: The $\HDETC$ algorithm}
The $\HDETC$ algorithm builds on the learning framework in Section \ref{routine:framework}, extending the Explore-then-Commit (\texttt{ETC}) method in bandits literature. To follow this method, agents uniformly query hints for covering matchings \( R \in \mathcal{R} \) , Lines~\ref{line:epo3}--\ref{line:epo4}, until time step \( T^\HDETC_0 \), determined by the assumed knowledge of \(\Delta^\text{match}_{\min}\).

At \( T^\HDETC_0 \) where $\rho$ is the index of the last exploration epoch, agents run \(\HALG\left(\bm{\tilde{\mu}}^\rho\right)\) to identify the matching \( G'^* \), which they commit to for all \( t > T^\HDETC_0 \), i.e., \( G(t) = G'^* \) (Lines~\ref{line:epo10}--\ref{line:epo11}). Theorem \ref{thm:hdetc} establishes that with a properly chosen \( T^\HDETC_0 \), which depends on \(\Delta^\text{match}_{\min}\), the algorithm achieves time-independent exploration regret while ensuring asymptotically optimal hint and communication usage. Detailed proofs are provided in Appendix~\ref{proof:hdetc}.

\begin{figure*}[t]
        \centering
        \begin{subfigure}[b]{0.99\textwidth}
            \centering\includegraphics[width=0.5\textwidth]{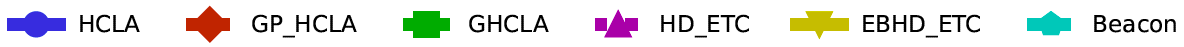}
            % \caption[Network2]%
            % {{\small Network 1}}    
            % \label{fig:mean and std of net14}
        \end{subfigure}
        \\
        \setcounter{subfigure}{0}
        \begin{subfigure}[b]{0.23\textwidth}
            \centering
            \includegraphics[width=\textwidth]{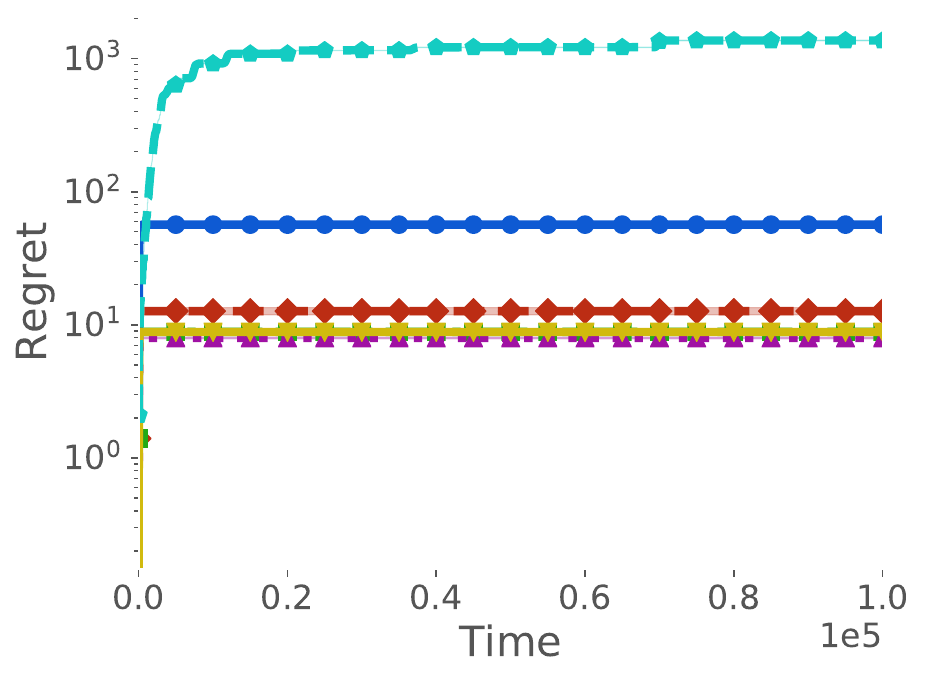}
            \caption[Regret]%
            {{\small Exploration Regret}}      
            \label{fig:exp}
        \end{subfigure}
        \hfill
        \begin{subfigure}[b]{0.23\textwidth}  
            \centering 
            \includegraphics[width=\textwidth]{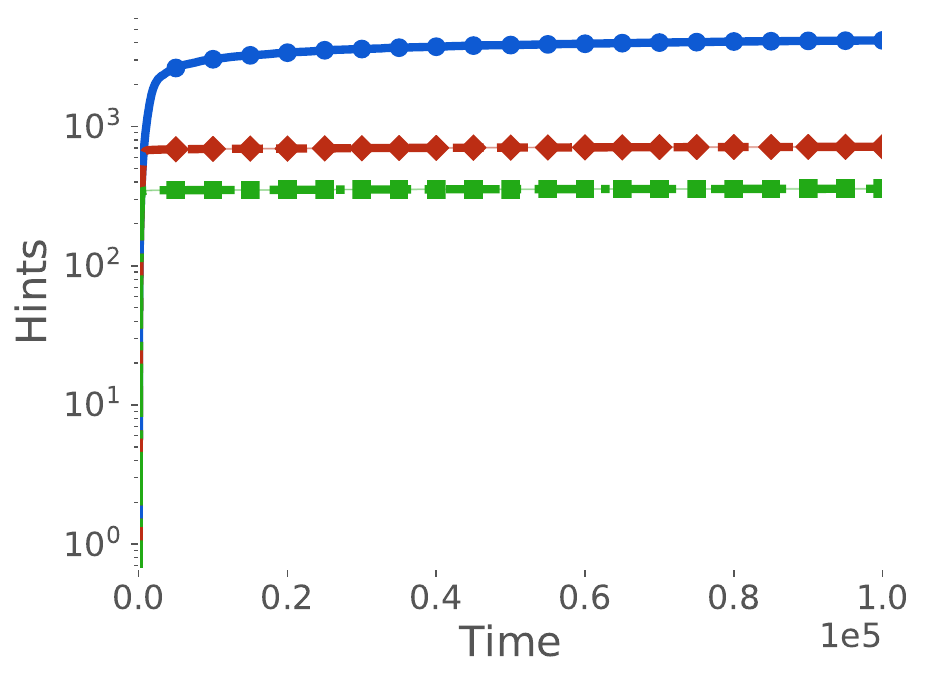}
            \caption[]%
            {{\small Centralized Hint Complexity}}    
            \label{fig:chints}
        \end{subfigure}
        \hfill
        \begin{subfigure}[b]{0.23\textwidth}   
            \centering 
            \includegraphics[width=\textwidth]{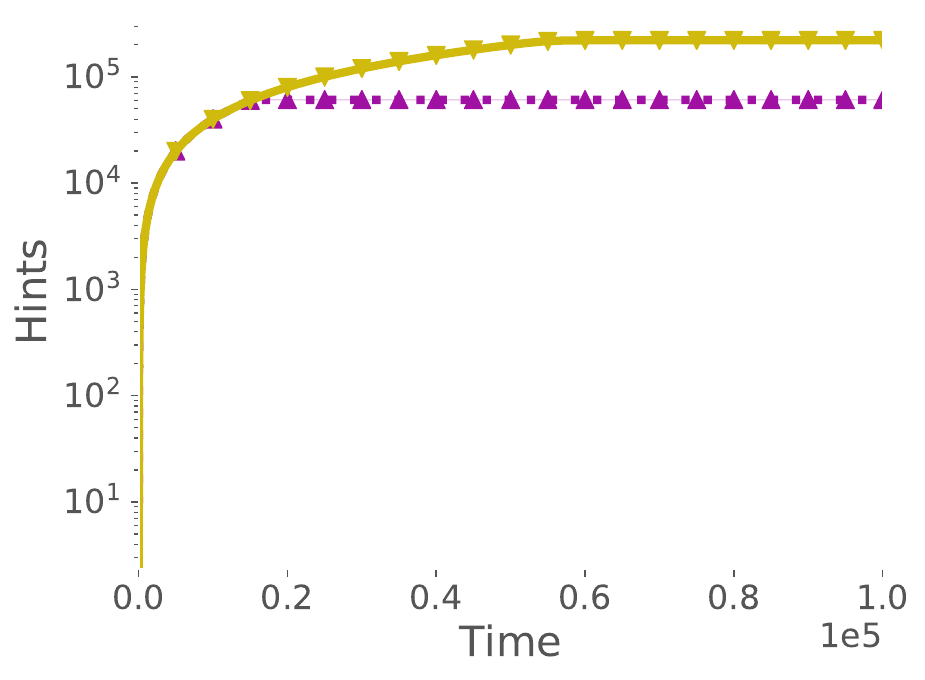}
            \caption[]%
            {{\small Decentral. Hint Complexity}}    
            \label{fig:dhints}
        \end{subfigure}
        \hfill
        \begin{subfigure}[b]{0.23\textwidth}   
            \centering 
            \includegraphics[width=\textwidth]{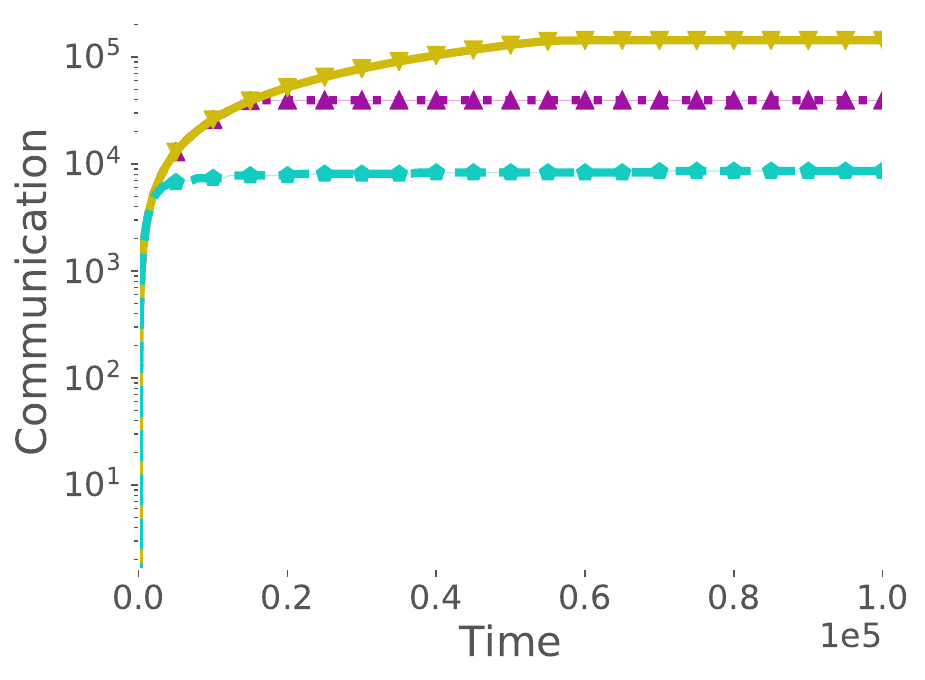}
            \caption[]%
            {{\small Communication Regret}}    
            \label{fig:dcom}
        \end{subfigure}
        \caption[ ]
        {{\small Figure \ref{fig:exp} plots $R^\pi(T)$ and $R^{\pi_{\text{exp}}}(T)$ for both centralized and decentralized setups. Figures \ref{fig:chints} and \ref{fig:dhints} reflects the $L^\pi(T)$ for centralized and decentralized algorithms respectively and Figure \ref{fig:dcom} shows $R^{\pi_{\text{com}}}(T)$ for decentralized algorithms. }}
        \label{fig:figss}
    \end{figure*}

\begin{algorithm}[t]
    \caption{ Hinted Decentralized Explore then Commit ($\HDETC$) : agent $m$}\label{routin:hdetc}
    \begin{algorithmic}[1]
        \Input agent \(m\), agent set $\mathcal{M}$, arm set \(\mathcal{K}\), number of agents \(M\), time horizon $T$, time threshold for hint inquiry $T^\HDETC_0$
        \State \textbf{Initialization:} $t \gets 0$, $\rho \gets 0$,
         $\bm{\hat{\mu}}_{m}(t)\gets 0$, $ \bm{N}^\HDETC_{m}(t) \gets 0,$
         $\bm{\tilde{\mu}}_{m'}^\rho \gets 0$ for each $m' \in \mathcal{M}$

        \While{$t < T^\HDETC_0$}
        \For{each epoch \(\rho\) }

        \State $G^\rho \gets \HALG\left(\bm{\tilde{\mu}}^\rho\right)$\label{line:epo2}
        \State $ t_0 \gets t$ 
        \LeftComment{Exploration Phase}
        \For{\(t \le t_0 + K\)}\label{line:hd-etc-reuse-begin}
        \State Pull the arm $k^{G^{\rho}}_m$
        {
        \footnotesize
        \State Update $\hat{\mu}_{m,k^{G\rho}_m} (t+1)$ and  $N^\HDETC_{m,k^{G\rho}_m} (t+1)$}
        \LeftComment{Hint Inquiry}
        \State $\ghint(t) \gets R_{(t\%K)+1}$\label{line:epo3}
        \State Ask for a hint from $k^{\ghint(t)}_m$ 
         {
        \footnotesize
        \State Update $\hat{\mu}_{m,k^{\ghint(t)}_m} (t+1)$ and $N^\HDETC_{m,k^{\ghint(t)}_m} (t+1)$} \label{line:epo4}

        \State $t \leftarrow t + 1$ 
        \EndFor
        \LeftComment{Communication Phase}
        \For{ $k \in [K]$ }\label{line:epo5}

        \State $\tilde{\delta}_{m, k}^{\rho+1} \gets \tilde{\mu}^{\rho+1}_{m,k} - \tilde{\mu}^{\rho}_{m,k}$
        \State $\operatorname{Send2All(\tilde{\delta}_{m, k}^{\rho+1})}$

        \EndFor
        
        \State $t \gets t + t_{\text{com}}^{\rho+1}$ \label{line:hd-etc-reuse-end}
        \State $\rho \gets \rho + 1$
        \EndFor
        \EndWhile
        \LeftComment{Exploitation Phase}
        \State $G'^* \gets  \HALG\left(\bm{\tilde{\mu}}^\rho\right)$ \label{line:epo10}
        \While{$t \leq T$}
        \State Pull the arm $k^{G'^*}_m$ \label{line:epo11}
        \EndWhile
    \end{algorithmic}
\end{algorithm}

% \begin{assumption}\label{asm:minimial-delta}
    % The algorithm knows the minimum gap \(\Delta^\text{match}_{\min}\).
% \end{assumption}
{
\footnotesize
\begin{restatable}{theorem}{hdetcp} \label{thm:hdetc}
     Assuming knowing the minimum gap \(\Delta^\text{match}_{\min}\), for the policy $\pi=\HDETC$ and $T^\pi_0 = \frac{9M^2K \log(2MT)}{(\Delta^\text{match}_{\min})^2}$, the policy $\pi$ has
    \begin{enumerate}
        \item exploration regret $R^{\pi_{\text{exp}}}(T) \in O\left(M^3K^2\right)$.
         \item hint complexity $L^{\pi}(T) \in O\left(MT^\pi_0\right)$
        \item communication regret $R^{\pi_{\text{com}}}(T) \in O\left(M^2T^\pi_0\right)$.
       
    \end{enumerate}

\end{restatable}
}
Although $\HDETC$ performs well, it assumes agents know \(\Delta^\text{match}_{\min}\), an unrealistic requirement in many settings. To overcome this, we propose $\EBHDETC$, an elimination-based algorithm that achieves similar bounds without relying on the minimum gap. The simplicity of the Explore-then-Commit structure in $\HDETC$ necessitates knowledge of \(\Delta^\text{match}_{\min}\) to determine the stopping time \(T^\HDETC_0\). Removing this assumption requires a more advanced algorithm design. In the next section, we introduce an elimination-based approach within the decentralized learning framework that operates without this gap assumption.

\subsection{The $\EBHDETC$ algorithm }

In $\EBHDETC$, agents transition into the exploitation phase differently compared to$\HDETC$. Accordingly, each agent maintains a set of active edges $\mathcal{C}^\rho$, which includes edges likely to be in $G^\rho$ for the upcoming epoch $\rho$, initially $\mathcal{C}^0 = \{(m,k) \in \mathcal{M}\times \mathcal{K}\}$. The set $\mathcal{C}^\rho$ maintained by each agent is the same across the others, determined by $\bm{\tilde{\mu}}$. Agents enter the exploitation phase when \(\abs{\mathcal{C}^\rho} = M\) (Line~\ref{line:epo101}). Edge removal from \(\mathcal{C}^\rho\) is guided by the error variable \(\epsilon^\rho\), defined as $\epsilon^\rho \coloneqq \sqrt{\frac{\log\left(\frac{2}{\eta}\right)}{\rho}}$, with $\eta = \sqrt{\frac{2}{MT^2}}$. At the beginning of each epoch, agents determine \( G^\rho \) and check whether there exists a matching \( G^\rho_{(m,k)} \neq G^\rho \) such that  
\begin{align}
U(G^\rho, \bm{\tilde{\mu}}^\rho) - U(G^\rho_{(m,k)}, \bm{\tilde{\mu}}^\rho) \leq 4M \bm{\epsilon}^\rho, \label{ineq:removal} 
\end{align} 
for all \((m,k) \in \mathcal{C}^\rho\). Each \( G^\rho_{(m,k)} \) is constructed as  
\( 
G_{(m,k)}^\rho \coloneqq \{(m,k)\} \cup \HALG\left(\bm{\tilde{\mu}}^\rho_{-m,-k}\right),  
\) 
where \(\bm{\tilde{\mu}}^\rho_{-m,-k}\) is the matrix \(\bm{\tilde{\mu}}^\rho\) with row \( m \) and column \( k \) removed. Agents update \(\mathcal{C}^{\rho+1}\) by removing edges \((m,k)\) for which \( G^\rho_{(m,k)} \) does not satisfy \eqref{ineq:removal} (Lines~\ref{line:epo120}--\ref{line:epo121}).This process continues until the stopping condition \(\abs{\mathcal{C}^\rho} = M\) is met. At that point, agents transition to the exploitation phase, committing to \( G'^* \), which corresponds to either \(\HALG\left(\bm{\tilde{\mu}}^\rho\right)\) or the matching formed by the edges in \(\mathcal{C}^\rho\) (both refer to the same matching) for the remainder of the rounds (Lines~\ref{line:epo130}--\ref{line:epo131}).

\begin{algorithm}[tb]
    \caption{ Elimination-Based Hinted Decentralized Explore then Commit ($\EBHDETC$) : agent $m$}\label{routin:ebhdetc}
    \begin{algorithmic}[1]
         \Input agent \(m\), agent set $\mathcal{M}$, arm set \(\mathcal{K}\), number of agents \(M\), time horizon $T$
        \State \textbf{Initialization:} $t \gets 0$, $\rho \gets 0$,
         $\bm{\hat{\mu}}_{m}(t)\gets 0$, $ \bm{N}^\HDETC_{m}(t) \gets 0,$
         $\bm{\tilde{\mu}}_{m'}^\rho \gets 0$ for each $m' \in \mathcal{M}$, $\mathcal{C}^\rho \gets \{(m,k) \in \mathcal{M} \times \mathcal{K}\}$

        % , communication bit number $ b_{m, k}^\rho \leftarrow \left\lfloor\log _2\left(N_{m,k}(t)\right)\right\rfloor$, for all arms $k$

        \While{$\abs{\mathcal{C}^{\rho}} > M$} \label{line:epo101}
        %\State $ b_{m, k}^\rho \leftarrow \left\lfloor\log _2\left(N_{m,k}(t)\right)\right\rfloor$, for all arms $k$

        \State $G^\rho \gets \HALG\left(\bm{\tilde{\mu}}^\rho\right)$

        \State Execute Lines~\ref{line:hd-etc-reuse-begin}-\ref{line:hd-etc-reuse-end} of Algorithm~\ref{routin:hdetc}
        % \LeftComment{Arm pulling}
        % \For{$K$ rounds}
        % \State Pull the arm $k^{G^{\rho}}_m$
        % \State Update $\hat{\mu}_{m,k^{G\rho}_m(m)}(t+1)$
        % \LeftComment{Hint query}
        % \State Ask for a hint from $k^{R_{t\%K}}_m$
        % \State Update $\hat{\mu}_{m,k^{R_{t\%K}}_m}(t+1)$

        % \State $t \leftarrow t + 1$
        % \EndFor 
        % \LeftComment{Communicate and update mean estimates}
        % \For{ $k \in [K]$ } 

        % \State $\tilde{\delta}_{m, k}^\rho \leftarrow \tilde{\mu}^{\rho}_{m,k} - \tilde{\mu}^{\rho-1}_{m,k}$
        % \State $\operatorname{Send2All(\tilde{\delta}_{m, k}^\rho)}$

        % \EndFor
        % \State $\rho \gets \rho + 1$
        % \State $t \gets t + \text{Communication 
        %  Time}$
         \State $ \mathcal{C}^{\rho+1} \gets \emptyset$
         \LeftComment{Updating Active Edges Set}
        \For{ $(m,k) \in \mathcal{C}^{\rho}$} \label{line:epo120}
        \State $G_{(m,k)}^\rho \gets (m,k) \cup \HALG\left(\bm{\tilde{\mu}}^\rho_{-m,-k}\right)$
        \If{$U(G^\rho; \bm{\tilde{\mu}}^\rho) - U(G^\rho_{(m,k)}; \bm{\tilde{\mu}}^\rho) \leq 4M \bm{\epsilon}^\rho$}
        \State $ \mathcal{C}^{\rho+1} \gets \mathcal{C}^{\rho+1} \cup (m,k)$ \label{line:epo121}
        \EndIf
        \EndFor
        \State $\rho \gets \rho + 1$

        \EndWhile
        \LeftComment{Exploitation Phase}
        \State $G'^* \gets  \HALG\left(\bm{\tilde{\mu}}^\rho\right)$ \label{line:epo130}\Comment{or $G'^* \gets \mathcal{C}^\rho$}
        \While{$t \leq T$}
        \State Pull the arm $k^{G'^*}_m$ \label{line:epo131}
        \EndWhile
    \end{algorithmic}
\end{algorithm}

{
    \footnotesize
\begin{restatable}{theorem}{ebhdetcp} \label{thm:ebhdetc}
    The policy $\pi=\EBHDETC$ has
    
    \begin{enumerate}
        \item exploration regret $R^{\pi_{\text{exp}}}(T) \in O\left(M^3K^2\right)$.
        \item hint complexity $L^{\pi}(T) \in O\left(\frac{M^3K\log\left(T\sqrt{2M}\right)}{\left(\Delta^\text{match}_{\min}\right)^2}\right)$
        \item communication regret  
        \(R^{\pi_{\text{com}}}(T) \in O\left(\frac{M^4K\log\left(T\sqrt{2M}\right)}{\left(\Delta^\text{match}_{\min}\right)^2}\right).\)
    \end{enumerate}
    
\end{restatable}
}

Theorem \ref{thm:ebhdetc} proves that $\EBHDETC$ achieves mostly similar but slightly weaker bounds in terms of coefficients compared to $\HDETC$, even without the assumption of knowing \(\Delta^\text{match}_{\min}\). A proof is given in Appendix~\ref{proof:ebhdetc}.

\section{Experiments}
We executed the algorithms $\HCLA$, $\GPHCLA$, $\GHCLA$, $\HDETC$, and $\EBHDETC$ with \( M = 4 \), \( K = 4 \), and \(\Delta^\text{match}_{\min} \leq 0.18\), averaging regret and hint complexity over 50 replications for \( 10^5 \) rounds. We also benchmarked the \texttt{Beacon} algorithm~\citep{shi2021heterogeneous} and compared its regret \( R^{{\texttt{Beacon}}_\text{exp}}(T) \) with ours in Fig.~\ref{fig:exp}, showing that our hint-augmented algorithms significantly reduce regret. Figs.~\ref{fig:chints} and \ref{fig:dhints} demonstrate that our centralized algorithms achieve better hint-query efficiency than $\HDETC$ and $\EBHDETC$ while maintaining similar regret. Fig.~\ref{fig:dcom} highlights \texttt{Beacon}'s advantage in communication regret due to its exponentially growing epoch lengths. These results confirm the effectiveness of centralized algorithms in balancing regret and hint complexity. Further extended experiments are presented in Appendix~\ref{sec:experiment}.

\section{Conclusion}
In this paper,we studied how hints enhance learning in the \HHMMAB\ problem, with heterogeneous rewards and collisions. We proposed both centralized and decentralized algorithms, analyzing their regret and hint usage. An interesting future work is extending these methods to two-sided matching markets, where both sides have preferences, and ties are broken due to them while collision occurs. This could enable the design of decentralized algorithms without communication to learn matchings with the same regret and hint optimality.

% In this work, we introduced $\HCLA$, a centralized algorithm for $\HHMMAB$, which generalizes an optimal algorithm for the homogeneous setup, achieving time-independent regret of $O\left(MK^{2M}\right)$ using $O\left(MK^M\log T\right)$ hints. We then proposed $\GPHCLA$, which achieves exponentially lower regret of $O\left(M^4K\right)$ with asymptotically optimal hint complexity of $O\left(MK\log T\right)$. In the more complex decentralized setup, where agents share statistics to avoid collisions while querying hints, we developed $\HDETC$ and $\EBHDETC$. These algorithms use an Explore-then-Commit approach, alternating between communication and exploration until the exploitation phase. $\HDETC$ relies on knowing $\Delta^\text{match}_{\min}$, while $\EBHDETC$ uses an elimination-based approach, with both achieving $O\left(M^3K^2\right)$ regret and optimal hint complexity of $O\left(M^3K\log (MT)\right)$. These algorithms successfully separate exploration and exploitation in $\HHMMAB$s, allowing for the effective use of hints to reduce regret. This framework can extend to more complex hint models and objectives. Future work may explore $\HHMMAB$s in areas like online scheduling, fair allocation, and double-sided matching markets.

\newpage
\section*{Acknowledgments}
The work of Mohammad Hajiesmaili is supported by NSF CNS-2325956, CAREER-2045641, CPS-2136199, CNS-2102963, and CNS-2106299. The work of Jinhang Zuo was supported by CityU 9610706. Xuchuang Wang is the corresponding author.
\bibliography{abb,aaai25}

\clearpage
\onecolumn
\appendix

% \section*{APPENDIX}

\section{The $\GHCLA$ Algorithm}\label{section:ghcla}
Here we analyse hint complexity of algorithm $\GHCLA$ that is a variation of $\HCLA$. In this algorithm a matching $G'(t)$ satisfying $U(G'(t);d(t)) > U(G(t);\hat{\mu}(t))$ would be served as $\ghint(t)$ with probability $\nicefrac{1}{2}$. 
\begin{algorithm}[H]
    \caption{ Generalized Centralized Learning Algorithm ($\GHCLA$)}\label{routin:hcla2}
    \begin{algorithmic}[1]
        % \Input set of agents $\mathcal{M}$, set of arms  $\mathcal{K}$, time horizon $T$

       \Input agent set $\mathcal{M}$, arm set \(\mathcal{K}\), time horizon \(T\), 
        \State \textbf{Initialization:} $t\gets 0$,  $\hat{\bm \mu}_{m}(t)\leftarrow \bm 0$,  $\bm{d}_{m}(t)\leftarrow \bm 0$,  $\bm{N}^{\GHCLA}_{m}(t)\gets \bm 0$ for each agent $m \in \mathcal{M}$

        \For{$t \in T$}

        \State $G(t) \leftarrow \HALG\left(\bm{\hat{\mu}}(t)\right)$

        \State $G'(t) \leftarrow \HALG\left(\bm{d}(t)\right)$
        \If{$U(G'(t),\bm{d}(t)) > U(G(t),\bm{\hat{\mu}}(t))$}

        \State $\ghint_1(t) \gets G'(t)$
        \State $\ghint_2(t) \gets$  pick a matching out of $\mathcal{R}$ uniformly at random \label{line:hint-choose-2}
        \State $\ghint(t) \gets
            \begin{cases}
                \ghint_1(t), & \text{w.p. } \frac{1}{2} \\
                \ghint_2(t), & \text{w.p. } \frac{1}{2}
            \end{cases}$
        \State Each agent $m$ asks for a hint from $k^{\ghint(t)}_m$
        \State Update $d_{G}(t+1)$ for all $G\in\mathcal{G}$\label{line:update-mu-kl-ucb}

        % \State Update
        \EndIf
        \State Each agent $m$ pulls $k^G_m$
        \State Update $\bm{\hat{\mu}}_m(t+1)$, \(\bm N^{\GHCLA}_{m}\left(t+1\right)\), and $\bm {d}_m(t+1)$ for each agent $m$ according to new observations

        \EndFor
    \end{algorithmic}
\end{algorithm}

\begin{restatable}{theorem}{ghclap}\label{lemma:hcla2}
    For $0< \delta < \frac{\Delta^\text{match}_{\min}}{2}$ and policy $\pi = \GHCLA$, the policy $\pi$ has
    \begin{enumerate}
        \item time-independent regret $R^\pi(T) \in O\left(M^4K\right)$,
        \item hint complexity $L^{\pi}(T) \in O\left(\frac{M^2K \log T}{\Delta^{\mathrm{kl}}}\right)$,
    \end{enumerate}
    where $\Delta^{\mathrm{kl}} = \mathrm{kl}(U(G^*;\bm{\mu}) - \Delta^\text{match}_{\min} + \delta, U(G^*;\bm{\mu}) - \delta)$.
\end{restatable}

‌By this Theorem, we prove that $\GHCLA$ asks for at most \(O\left(\frac{M^2K \log T}{\Delta^{\mathrm{kl}}}\right)\) hints in the worst case which is $M$ times larger than that of $\GPHCLA.$

\section{Analysis of \(\HCLA\),$\GPHCLA$, and $\GHCLA$}\label{eba}
Our proofs use an \emph{event-based} analysis inspired by \citet{wang2020optimal}, which was originally developed for homogeneous multi-armed bandits. We have adapted this method for the heterogeneous setting, addressing its unique challenges. We will now explain the Event-Based Regret Analysis and how it applies to regret and hint complexity analysis of $\HCLA$, $\GPHCLA$, and $\GHCLA$. The regret analysis for the decentralized algorithms $\HDETC$ and $\EBHDETC$ follows the same approach, but their detailed analysis is deferred to Appendix \ref{sec:deba}. This method mainly focuses on bounding the number of times specific events that affect regret occur, as defined below.

\begin{restatable}{definition}{events}
    \label{def:sets}
    We define the following sets of time steps at which specific events occur as:

    \begin{align*}
        \mathcal{A} & \coloneqq \{t \geq 1: G(t) \neq G^*\}                                                                                                                                    \\
        \mathcal{B} & \coloneqq \{t \geq 1: \abs{U\left(G(t); \bm{\hat{\mu}}(t)\right) - U\left(G(t); \bm{\mu}\right)} \geq \delta\}                                                           \\
        \mathcal{C} & \coloneqq \{t \geq 1: U\left(G^*; \bm{d}(t)\right) < U\left(G^*; \bm{\mu}\right)\}                                                                                       \\
        \mathcal{D} & \coloneqq \{t \geq 1: t \in \mathcal{A} \setminus (\mathcal{B} \cup \mathcal{C}), \abs{U\left(G^*; \bm{\hat{\mu}}(t)\right) - U\left(G^*; \bm{\mu}\right)} \geq \delta\} \\
        \mathcal{E} & \coloneqq \{t \geq 1: U(G'(t), \bm{d}(t)) > U(G^*, \hat{\bm{\mu}}(t))\}
    \end{align*}
\end{restatable}

Intuitively, \(\mathcal{A}\) contains the time steps at which a sub-optimal matching is pulled, incurring regret. We also need to bound the size of the set \(\mathcal{B}\), which includes the time steps where the chosen matching is not well-estimated, potentially leading to the selection of a sub-optimal \(G(t)\). The sets \(\mathcal{C}\) and \(\mathcal{D}\) help us bound the exploration needed to identify \(G^*\) through hints when a sub-optimal matching is selected. Finally, \(\mathcal{E}\) bounds the number of required hints when \(G^*\) is detected, to explore potential matchings that might not have been chosen as \(G(t)\) due to potentially inaccurate estimations.

We can observe that for a policy $\pi \in \{\HCLA,\GHCLA,\GPHCLA\}$,
\[ R^\pi(T) \leq M \E\left[\abs{\mathcal{A}}\right], \]
so we focus on bounding \( \E\left[\abs{\mathcal{A}}\right] \), which represents the expected number of time steps during which the central decision maker selects a sub-optimal matching under each policy.

Next, we prove that Lemma \ref{lemma:proof1}, originally established in \citet{wang2020optimal} for homogeneous multi-agent multi-armed bandits, also holds for the \(\HHMMAB\)s. This is done by substituting each arm with a matching (i.e., a super arm) in their proof, without loss of generality.

\begin{lemma}[\citet{wang2020optimal}]\label{lemma:proof1}
    For the \(\HHMMAB\)s, the following chain of inequalities is satisfied:
\begin{align} \label{thm:cphi_events}
    \E\left[\abs{\mathcal{A}}\right] \leq \E\left[\abs{\mathcal{A} \cup \mathcal{B}}\right] \leq \E\left[\abs{\mathcal{B}}\right] + \E\left[\abs{\mathcal{C}}\right] + \E\left[\abs{\mathcal{D}}\right].
\end{align}

\end{lemma}
In analyzing the algorithms $\HCLA$, $\GPHCLA$, and $\GHCLA$, we focus on bounding each term on the RHS of inequality \eqref{thm:cphi_events} separately to assess the regret. The final step in the proofs is to bound the expected size of $\E\left[\abs{\mathcal{B}}\right]$, which determines the asymptotic hint complexity $L^\pi(T)$, where $\pi \in \{\HCLA, \GPHCLA, \GHCLA\}$.

First, we analyze the warm-up algorithm $\HCLA$ by proving Theorem \ref{thm:matirphi} utilizing matching-level analysis, which is so straightforward compared to the proofs for $\GPHCLA$ and $\GHCLA$ that depend on multiple edge-level lemmas that we prove later. We then present the proof of Lemmas \ref{lemma:proof2}, \ref{lemma:proof3}, and \ref{lemma:proof4},  which bound $\E\left[\abs{\mathcal{B}}\right]$, $\E\left[\abs{\mathcal{C}}\right]$, and $\E\left[\abs{\mathcal{D}}\right]$ respectively, to analyze the regret of $\GPHCLA$ and $\GHCLA$ using edge-level analysis. Unlike \(\HCLA\), this method takes into account how observations from one matching influence the estimations of other matchings that share an edge or a subset of edges. Moreover, $\GPHCLA$ and $\GHCLA$ differ in the worst case upperbound for $\E\left[\abs{\mathcal{E}}\right]$. By Lemma \ref{lemma:proof5}, we bound the size of $\mathcal{E}$ for $\GPHCLA$ and move that of $\GHCLA$ to the proof of Theorem \ref{lemma:hcla2}.

\subsection{Proof of Theorem \ref{thm:matirphi}} \label{proof:hcla}
\hclap*
\begin{proof}
    We bound the size of the sets of time steps defined in Definition \ref{def:sets} as follows:

    By Lemma \ref{lemma:obs}, we can prove that $\E\left[|\mathcal{B}|\right] \leq 4+2\delta^{-2}$. Bounding $\E\left[|\mathcal{C}|\right]$ is also a straightforward application of Lemma \ref{lemma:KLh} by replacing $(m,k) \in \mathcal{M} \times \mathcal{K}$ with $G\in\mathcal{G}$ to have $\E\left[|\mathcal{C}|\right] \leq 15$. We can also prove that at every time step $t\in\mathcal{D}$, $U(G^*;\bm{d}(t))>U(G(t);\bm{\mu}(t))$, thus decision maker hints at that time under $\HCLA$ algorithm. It the hints the $G^*$ with probability at least $\nicefrac{1}{2K^M}$, which implies $\E\left[|\mathcal{D}|\right] \leq K^M(K^M+\delta^{-2})$ by applying Lemma \ref{lemma:obs}. Therefore, we can write
    \begin{align*}
        R^\HCLA(T) & \overset{(a)}\leq M\left(\E\left[|\mathcal{B}|\right] + \E\left[|\mathcal{C}|\right] + \E[|\mathcal{D}|]\right)\label{lemma:phi_events}, \\
                   & \overset{(b)}\leq 4M + 2M\delta^{-2} + 15M + MK^{2M}+MK^M\delta^{-2}.
    \end{align*}
    where inequality (a) is directly adopted from \citet{wang2020optimal} in which the authors showed that $\mathcal{A} \cup \mathcal{B} \subset \mathcal{B} \cup \mathcal{C} \cup \mathcal{D}$.

    We next prove an upper bound on the $L^\HCLA(T)$ by bounding the number of time steps the condition in Line~\ref{l:line:klh} of the $\HCLA$ occurs. For a fixed matching $G' \in \mathcal{G}$, we define $\mathcal{E}_{G'}$ as
    $$
        \mathcal{E}_{G'}\coloneqq\left\{t \in \mathcal{E}:G'(t) = G', \bm{d}(t)>\bm{\hat{\mu}}(t)\right\}.
    $$

    Now, we can write
    \begin{align*}
        L^\HCLA(T) & \leq M\sum_{G' \in \mathcal{G}} \mathbb{E}\left[\abs{\mathcal{E}_{G'}}\right],                                          \\
                   & \overset{(a)}\leq \frac{2MK^M\left(\log T+4 \log \log T\right)}{\Delta^{\mathrm{kl}}}+MK^M\left(4+2 \delta^{-2}\right),
    \end{align*}
    The inequality (a) is obtained by applying Lemma \ref{lemma:obs} to the established argument in \cite{wang2020optimal}, noting that under \(\HCLA\), the probability of requesting a hint for \(\ghint(t) = G'\) is at least \(\nicefrac{1}{2}\) at those hint inquiry time steps. For further details, readers can refer to the Proof of Lemma \ref{lemma:proof5}, where the same argument is validated using edge-level analysis. We skip the detailed proof here as a warm-up for our main algorithm \(\GPHCLA\).

Finally, we conclude that \(R^\HCLA(T) \in O\left(MK^{2M}\right)\) and \(L^\HCLA(T) \in O\left( \frac{MK^M\log T}{\Delta^{\mathrm{kl}}}\right)\).

\end{proof}

\subsection{Proof of Theorem \ref{thm:c_whhmmab}}\label{proof:gphcla}
\gphclap*

\begin{proof}

    Using Lemma~\ref{lemma:proof1} we can write:
    \begin{align*}
        R^\GPHCLA(T) & \leq M\left(\E\left[\abs{\mathcal{B}}\right]+\E\left[\abs{\mathcal{C}}\right]+\E\left[\abs{\mathcal{D}}\right]\right), \\
                     & \overset{(a)}\leq 15M^2 + 4M^2K + 16M^2k^2 + 6M^4K\delta^{-2},
    \end{align*}
    where inequality (a) holds by aggregation over the results of Lemmas~\ref{lemma:proof2},~\ref{lemma:proof3}, and~\ref{lemma:proof4}.
    Therefore, the total regret of \(\GPHCLA\) is included to be in \(O\left(M^4K\right)\).

    Now that we bounded $\E\left[\abs{\mathcal{B}}\right] + \E\left[\abs{\mathcal{C}}\right] + \E\left[\abs{\mathcal{D}}\right]$, we bound the hints that $\GPHCLA$ queries for after detecting the $G^*$. By directly applying the result of Lemma~\ref{lemma:proof5}, we can write
    \begin{align*}
        L^\GPHCLA(T) & \leq M \sum_{G' \in \mathcal{G}} \E\left[\abs{\mathcal{E}_{G'}}\right],  \\
                     & \leq 4MK^3 + 2M^3K\delta^{-2} + \frac{2MK \log T}{\Delta^{\mathrm{kl}}},
    \end{align*}
    which implies that $L^\GPHCLA(T) \in O\left(\frac{MK \log T}{\Delta^{\mathrm{kl}}}\right)$ that wraps the proof up.
\end{proof}

\subsection{Detailed Lemmas Required to Analyze $\GHCLA$ and $\GPHCLA$}

\paragraph{Concentration bounds}
\hfill \break
Here we introduce lemmas play very central role in our event-based analysis.
\begin{lemma}[\citet{combes2015learning}, Lemma 6] \label{lemma:KLh}
    For any  $(m,k) \in \mathcal{M} \times \mathcal{K}$, we have
    $$
        \sum_{t \geq 1} \Pr\left[d_{m,k}(t) \leq \mu_{m,k}\right] \leq 15.
    $$
\end{lemma}

\begin{lemma}[\citet{wang2020optimal}, Lemma 9] \label{lemma:obs}
    Let $(m,k) \in \mathcal{M} \times \mathcal{K}$, and $c>0$. Let $H$ be a random set of rounds such that for all $t,\{t \in H\} \in \mathcal{F}_{t-1}$. Assume that there exists $\left(C_{m,k}(t)\right)_{t \geq 0}$, a sequence of independent binary random variables such that for any $t \geq 1, C_{m,k}(t)$ is $\mathcal{F}_t$-measurable and $\Pr\left[C_{m,k}(t)=1\right] \geq c$. In addition, assume that for any $t \in H$, agent $m$ pulls the arm $k$ if $C_{m,k}(t)=1$. Then,
    $$
        \sum_{t \geq 1} \Pr\left[\left\{t \in H,\left|\hat{\mu}_{m, k}(t)-\mu_{m,k}\right| \geq \varepsilon\right\}\right] \leq 2 c^{-1}\left(2 c^{-1}+\varepsilon^{-2}\right) .
    $$
\end{lemma}

\paragraph{Event-Based Analysis Lemmas}
\hfill \break
Here, we prove the lemmas necessary to analyze the regret and hint complexity of $\GPHCLA$ using the event-based analysis described above.

\begin{lemma}\label{lemma:proof2}

    For policy $\pi = \{\GHCLA,\GPHCLA\}$, the expected size of set $\mathcal{B}$ is bounded by
    \[
        \E\left[\abs{\mathcal{B}}\right] \leq 4MK + \frac{2M^3K}{\delta^2}.
    \]

\end{lemma}
\begin{proof}
    To bound \(\E\left[\abs{\mathcal{B}}\right]\), we start by defining a new set of time steps as follows:
    \begin{align*}
        \mathcal{B}_{m,k} \coloneqq \left\{t \geq 1:  \abs{\hat{\mu}_{m,k}(t) - \mu_{m,k}(t)}\geq \frac{\delta}{M}\right\}.
    \end{align*}

    We now show that
    \[
        \E\left[\abs{\mathcal{B}}\right] \leq \sum_{(m,k) \in \mathcal{M} \times \mathcal{K}}\E\left[\abs{\mathcal{B}_{m,k}}\right].
    \]

    First, we can write:
    \begin{align*}
        \E\left[\abs{\mathcal{B}}\right] & = \E\left[\sum_{t \geq 1} \mathds{1}\left\{\abs{U(G(t); \bm{\hat{\mu}}(t)) - U(G(t); \bm{\mu})} \geq \delta\right\}\right], \\
         &\overset{(a)} \leq \E\left[\sum_{t \geq 1} \mathds{1}\left\{\sum_{(m,k) \in G(t)} \abs{\hat{\mu}_{m,k}(t) - \mu_{m,k}} \geq \delta\right\}\right],                                     \\
         & \leq \E\left[\sum_{t \geq 1} \mathds{1}\left\{\exists (m,k) \in G(t): \abs{\hat{\mu}_{m,k}(t) - \mu_{m,k}} \geq \frac{\delta}{M}\right\}\right],          \\
         & \overset{(b)}\leq \E\left[\sum_{t \geq 1} \sum_{(m,k) \in G(t)} \mathds{1}\left\{\abs{\hat{\mu}_{m,k}(t) - \mu_{m,k}} \geq \frac{\delta}{M}\right\}\right],                          \\
         & \leq \E\left[\sum_{t \geq 1} \sum_{(m,k) \in \mathcal{M} \times \mathcal{K}} \mathds{1}\left\{\abs{\hat{\mu}_{m,k}(t) - \mu_{m,k}} \geq \frac{\delta}{M}\right\}\right], \\
         & \leq \sum_{(m,k) \in \mathcal{M} \times \mathcal{K}} \E\left[\sum_{t \geq 1} \mathds{1}\left\{\abs{\hat{\mu}_{m,k}(t) - \mu_{m,k}} \geq \frac{\delta}{M}\right\}\right], \\
         & \leq \sum_{(m,k) \in \mathcal{M} \times \mathcal{K}} \E\left[\abs{\mathcal{B}_{m,k}}\right].
    \end{align*}

    In the above steps, inequality (a) holds due to the triangle inequality, which implies that \(\abs{U\left(G(t); \bm{\hat{\mu}}(t)\right) - U\left(G(t); \bm{\mu}\right)} \leq \sum_{(m,k)} \abs{\hat{\mu}_{m,k}(t) - \mu_{m,k}(t)}\). Inequality (b) is derived by applying the union bound.

    Therefore, by Lemma \ref{lemma:obs}, we can conclude:
    \begin{align*}
        \E[|\mathcal{B}|] & \leq \sum_{(m,k) \in \mathcal{M} \times \mathcal{K}} \E[|\mathcal{B}_{m,k}|], \\
                          & \leq 4MK + \frac{2M^3K}{\delta^2}.
    \end{align*}

\end{proof}

\begin{lemma}\label{lemma:proof3}
    For policy $\pi = \{\GHCLA,\GPHCLA\}$, the expected size of set $\mathcal{C}$ is bounded by
    \[
        \E\left[\abs{\mathcal{C}}\right] \leq 15M.
    \]

\end{lemma}
\begin{proof}
    To bound \(\E[|\mathcal{C}|]\), we apply the union bound on all edges contained in \(G^*\). Accordingly, we can write:
    \begin{align*}
        \E\left[\abs{\mathcal{C}}\right] & \leq \sum_{(m,k) \in G^*} \E\left[\abs{\{t \geq 1: d_{m,k}(t) < \mu_{m,k}\}}\right] \\
                                         & \leq 15M,
    \end{align*}
    where the second inequality follows from the results of Lemma \ref{lemma:KLh}.

\end{proof}

\begin{lemma}\label{lemma:proof4}
    For policy $\pi = \{\GHCLA,\GPHCLA\}$, the expected size of set $\mathcal{D}$ is bounded by
    \[
        \E\left[\abs{\mathcal{D}}\right] \leq 16MK^2 + \frac{4M^3K}{\delta^2}.
    \]
\end{lemma}
\begin{proof}
    First we can observe that at each time step  \(t \in \mathcal{D}\) the hint inquiry condition is satisfied for the optimal matching, i.e.,
    \[
        U\left(G^*;\bm{d}(t)\right) \geq U\left(G(t);\bm{\hat{\mu}}(t)\right),
    \]
    which means that a hint is queried by $\GPHCLA$ at all those time steps.

    Next, we define a new set of time steps as follows:
    \begin{align*}
        \mathcal{D}_{m,k} \coloneqq \{t \geq 1: t \in \mathcal{A} \setminus (\mathcal{B} \cup \mathcal{C}), |\hat{\mu}_{m,k}(t) - \mu_{m,k}(t)| \geq \frac{\delta}{M}\}.
    \end{align*}

    We can then write:

    \begin{align*}
        \E[|\mathcal{D}|] & \overset{(a)} \leq \sum_{(m,k) \in G^*} \E[|\mathcal{D}_{m,k}|], \\
                          & \overset{(b)}\leq 16MK^2 + \frac{4M^3K}{\delta^2},
    \end{align*}
    Inequality (a) is derived using the same method as in the proof of Lemma \ref{lemma:proof2} for bounding $\E\left[\abs{\mathcal{B}}\right]$. To achieve inequality (b), we utilize the fact that an edge $(m,k) \in G^*$ is independently hinted with a probability of at least $\nicefrac{1}{2k}$ at each time step in $\mathcal{D}$. Consequently, the number of time steps in $\mathcal{D}$ where $\abs{\hat{\mu}_{m,k} - \mu_{m,k}} \geq \frac{\delta}{M}$ holds for $(m,k)$ is bounded by Lemma \ref{lemma:obs}. We then apply the union bound to all the edges in $G^*$ to derive inequality (b).

\end{proof}

\begin{lemma}
    % [Adatped from~\citep[Sec 3.1]{shi2021heterogeneous}]
    \label{lemma:proof5}
    Let us define $\mathcal{E}_{G'}$ as be the set of time steps in $\mathcal{E}$ at which $\GPHCLA$'s hint inquiry condition is triggered by a matching $G'$:
    \[
        \mathcal{E}_{G'} \coloneqq \{t \in \mathcal{E} \mid G'(t) = G', U(G', \bm{d}(t)) > U(G^*, \hat{\bm{\mu}}(t))\}.
    \]
    Then, for \(0 < \delta < \frac{\Delta^\text{match}_{\min}}{2}\), the expected size of this set satisfies the following bound:
    \[
        \sum_{G' \in \mathcal{G}} \E\left[\abs{\mathcal{E}_{G'}}\right] \leq 4K^3 + 2M^2K\delta^{-2} + \frac{2K \log T}{\Delta^{\mathrm{kl}}},
    \]
    where $\Delta^{\mathrm{kl}} = \mathrm{kl}(U(G^*;\bm{\mu}) - \Delta^\text{match}_{\min} + \delta, U(G^*;\bm{\mu}) - \delta).$

\end{lemma}
\begin{proof}
    We start the proof defining new sets of time steps as:

    \[
        \mathcal{E}_{1,G'} \coloneqq \{t \in \mathcal{E}_{G'}: \abs{U(G';\bm{\hat{\mu}}(t)) - U(G';\bm{\mu}(t))} \geq \delta\},
    \]

    \[
        \mathcal{E}_{2,G'} \coloneqq \left\{t \in \mathcal{E}_{G'}: \sum_{t' = 1}^{t} \mathds{1}\left\{t' \in \mathcal{E}_{G'}\right\} < t_0 \right\}.
    \]
    where \(t_0 = \frac{\log T + 4 \log \log T}{\Delta^{\mathrm{kl}}}\).

    Similar to the argument \citet{wang2020optimal} made, we can prove that \[\E\left[\abs{\mathcal{E}_{G'}}\right] \leq \E\left[\abs{\mathcal{E}_{1,G'}}\right] + \E\left[\abs{\mathcal{E}_{2,G'}}\right]. \] Accordingly we bound each $\E\left[\abs{\mathcal{E}_{1,G'}}\right]$ and $\E\left[\abs{\mathcal{E}_{2,G'}}\right]$ separately and then apply union bound over it. However, applying the union bound on the set of all matchings \(\mathcal{G}\) to bound \(\E\left[\abs{\mathcal{E}_{G'}}\right]\) would result in an exponential constant, i.e., \(K^M\), for the number of hints policy $ \GPHCLA$ queries a hint for. To mitigate this, \(\GPHCLA\) utilizes the idea of projecting $G' \in \mathcal{G}$ to a matching in the sset of covering matchings \(\mathcal{R}\) to specify $\ghint(t)$, thereby reducing the possible set of hints from the exponentially-sized \(\mathcal{G}\) to \(\mathcal{R}\) that has $K$ matchings in it. We can observe that every time a matching \(\ghint(t)\) is hinted at by \(\GPHCLA\), the counter for all involved edges increases by one, i.e., \(\forall (m,k) \in \ghint(t), N^{\GPHCLA}_{m,k}(t+1) = N^{\GPHCLA}_{m,k}(t) + 1\). Additionally, under this projection-based hint inquiry approach, the number of rounds each arbitrary edge \((m,k) \in \ghint(t)\) is hinted at is equal to that of the other edges in \(\ghint(t)\).  
    
    Accordingly, we bound  $\E\left[\abs{\mathcal{E}_{1,G'}}\right]$ as:

        \begin{align*}
            \sum_{G' \in \mathcal{G}} \E\left[\abs{\mathcal{E}_{1,G'}}\right] & = \E\left[\sum_{G'\in\mathcal{G}}\sum_{t \in \mathcal{E}_{G'}} \mathds{1}\left\{\abs{U(G'(t);\bm{\hat{\mu}}(t)) - U(G'(t);\bm{\mu}(t))} \geq \delta\right\}\right],                                                                      \\
                                                                              & = \E\left[\sum_{G'\in\mathcal{G}}\sum_{t \in \mathcal{E}_{G'}} \mathds{1}\left\{\exists (m,k) \in G': \abs{\hat{\mu}_{m,k}(t)-\mu_{m,k}}\geq \frac{\delta}{M}\right\}\right],                                                            \\
                                                                              & \overset{(a)}= \E\left[\sum_{R\in\mathcal{R}}\sum_{t \in \cup_{G' \in \mathcal{G}}\mathcal{E}_{G'}} \mathds{1}\left\{\ghint(t) = R, \exists  (m,k)  \in G'(t),R: \abs{\hat{\mu}_{m,k}(t)-\mu_{m,k}}\geq \frac{\delta}{M}\right\}\right], \\
                                                                              & \leq \sum_{R\in\mathcal{R}}\E\left[\sum_{t \in \cup_{G' \in \mathcal{G}}\mathcal{E}_{G'}} \mathds{1}\left\{\exists  (m,k)  \in G'(t),R: \abs{\hat{\mu}_{m,k}(t)-\mu_{m,k}}\geq \frac{\delta}{M}\right\}\right],                          \\
                                                                              & \overset{(b)}\leq K(4K^2 + 2M^2\delta^{-2}),                                                                                                                                                                                             \\
                                                                              & = 4K^3 + 2M^2K\delta^{-2}.
        \end{align*}
    
    We do not need to take the union bound over all edges \((m,k) \in \mathcal{M}\times\mathcal{K}\) because \(\ghint\) is selected from the covering matchings \(\mathcal{R}\), ensuring that all involved edges are observed by the corresponding agents. Therefore, applying the union bound over the \(K\) pairwise edge-independent members of \(\mathcal{R}\) effectively accounts for all edges, leading to equality (a). Inequality (b) follows by applying Lemma \ref{lemma:obs}, considering that at each time step \(t \in \mathcal{E}_{1,G'}\), the probability of hinting a matching that includes the edge \((m,k) \in G'(t)\), where \(\abs{\hat{\mu}_{m,k}(t) - \mu_{m,k}} \geq \frac{\delta}{M}\), is at least \(\nicefrac{1}{2K}\).

    Now we bound $\E\left[\abs{\mathcal{E}_{2,G'}}\right]$ as:

        \begin{align*}
            \sum_{G' \in \mathcal{G}} \E\left[\abs{\mathcal{E}_{2,G'}}\right] & = \E\left[\sum_{G'\in\mathcal{G}}\sum_{t \in \mathcal{E}_{G'}} \mathds{1}\left\{\sum_{t' = 1}^{t} \mathds{1}\left\{t' \in \mathcal{E}_{G'}\right\} < t_0\right\}\right],                                                                       \\
                                                                              & = \E\left[\sum_{G'\in\mathcal{G}}\sum_{t \in \mathcal{E}_{G'}} \mathds{1}\left\{\exists  (m,k)  \in G':\sum_{t' = 1}^{t} \mathds{1}\left\{  (m,k)  \in G'(t')\right\} < t_0\right\}\right],                                                    \\
                                                                              & = \E\left[\sum_{G'\in\mathcal{G}}\sum_{t \in \mathcal{E}_{G'}} \mathds{1}\left\{\min_{(m,k) \in G'}\sum_{t' = 1}^{t} \mathds{1}\left\{  (m,k) \in G'(t')\right\} < t_0\right\}\right],                                                         \\
                                                                              & \overset{(a)}= \E\left[\sum_{R\in\mathcal{R}}\sum_{t \in \cup_{G' \in \mathcal{G}}\mathcal{E}_{G'}} \mathds{1}\left\{\min_{ (m,k)  \in G'}\sum_{t' = 1}^{t} \mathds{1}\left\{\ghint(t') = R, (m,k) \in G'(t'),R \right\} < t_0\right\}\right], \\
                                                                              & \leq \sum_{R\in\mathcal{R}}\E\left[\sum_{t \in \cup_{G' \in \mathcal{G}}\mathcal{E}_{G'}} \mathds{1}\left\{\min_{(m,k) \in G'}\sum_{t' = 1}^{t} \mathds{1}\left\{ (m,k) \in G'(t'),R\right\} < t_0\right\}\right],                             \\
                                                                              & \overset{(b)}\leq \frac{2K \log T}{\Delta^{\mathrm{kl}}}.
        \end{align*}

    Equality (a) holds because the number of times each matching \(G'\) is hinted at is always less than or equal to the number of times \(\ghint(t)\) covers each edge \((m,k) \in G'\), including the edge with the minimum number of hint occurrences. Inequality (b) holds because the probability of increasing the counter for any edge, specifically the edge with the fewest hint occurrences, is at least \(\nicefrac{1}{2}\). This implies that the expected number of trials to increase this counter is at most 2. By applying the union bound over all members of \(\mathcal{R}\), we obtain an upper bound on the total number of required hints.
    
    Finally we can imply that
    \begin{align*}
        \sum_{G' \in \mathcal{G}} \E\left[\abs{\mathcal{E}_{G'}}\right] & \leq \E\left[\abs{\mathcal{E}_{1,G'}}\right] + \E\left[\abs{\mathcal{E}_{2,G'}}\right], \\
                                                                        & \leq 4K^3 + 2M^2K\delta^{-2} + \frac{2K \log T}{\Delta^{\mathrm{kl}}}.
    \end{align*}
\end{proof}

\subsection{Proof of Theorem \ref{lemma:hcla2}} \label{proof:ghcla}
  The main difference between \(\GPHCLA\) and \(\GHCLA\) is the determining $\ghint$ as the same event is triggered. Thus, the regret analysis for both algorithms is the same, while the number of hints is \(M\) times larger for \(\GHCLA\) according to the following theorem.
\ghclap*
\begin{proof}
  
    We skip the analysis of $R^\pi(T)$ cause it follows the exact upper bound proven for $\GPHCLA$ in Theorem \ref{thm:c_whhmmab}.
    Now, to bound the hint complexity, we use the same set of time steps \(\mathcal{E}_{G'}\) previously defined in the proof of Theorem \ref{thm:c_whhmmab} and define new sets of time steps \( \mathcal{E}_{1,G'}\ \text{and } \mathcal{E}_{2,G'}\) as:

    \begin{align*}
        \mathcal{E}_{1,G'}  & \coloneqq \{t \in \mathcal{E}_{G'}: \abs{U\left(G';\bm{\hat{\mu}}(t)\right) - U\left(G';\bm{\mu}(t)\right)} \geq \delta\},          \\
        \mathcal{E}_{2,G'}  & \coloneqq \{t \in \mathcal{E}_{G'}: \sum_{t' = 1}^{t} \mathds{1}_{\{t' \in \mathcal{E}_{G'}\}} < t_0 \},                            \\
        \mathcal{E}'_{1,G'} & \coloneqq \left\{t \in \mathcal{E}_{G'}: \exists  (m,k) \in G', \abs{\hat{\mu}_{m,k}(t) - \mu_{m,k}} \geq \frac{\delta}{M}\right\}, \\
        \mathcal{X}_{1,m,k} & \coloneqq \left\{t \geq 1: ( m,k ) \in G'(t), \abs{\hat{\mu}_{m,k}(t) - \mu_{m,k}} \geq \frac{\delta}{M}\right\},                   \\
        \mathcal{X}_{2,m,k} & \coloneqq \left\{t \geq 1: \sum_{t' = 1}^{t} \mathds{1}\left\{(m,k) \in G'(t)\right\} < t_0\right\}.
    \end{align*}
    As argued in the proof of Theorem \ref{thm:c_whhmmab}, we know that \(\mathcal{E}_{G'} \subseteq \mathcal{E}_{1,G'} \cup \mathcal{E}_{2,G'} \).Then, we can conclude that
    $\E\left[\abs{\mathcal{E}_{1,G'}}\right] \leq \E\left[\abs{\mathcal{E}'_{1,G'}}\right],$ and also \[
        \sum_{G' \in \mathcal{G}}\E\left[\abs{\mathcal{E}_{2,G'}}\right] \leq \sum_{(m,k) \in \mathcal{M} \times \mathcal{K}} \E\left[\abs{\mathcal{X}_{2,m,k}}\right].
    \]

    Therefore, we finish the proof by writing:

    \begin{align*}
        \sum_{G' \in \mathcal{G}} \E\left[\abs{\mathcal{E}_{G'}}\right]
         & \leq \sum_{G' \in \mathcal{G}} \E\left[\abs{\mathcal{E}_{1,G'}}\right] + \E\left[\abs{\mathcal{E}_{2,G'}}\right],                                                                                                                                                      \\
         & \leq \sum_{G' \in \mathcal{G}} \E\left[\abs{\mathcal{E}'_{1,G'}}\right] + \E\left[\abs{\mathcal{E}_{2,G'}}\right],                                                                                                                                                     \\
         & \overset{(a)}{\leq} \E\left[\sum_{G' \in \mathcal{G}} \abs{\mathcal{E}'_{1,G'}}\right] + \sum_{G' \in \mathcal{G}} \E\left[\abs{\mathcal{E}_{2,G'}}\right],                                                                                                            \\
         & \overset{(b)}{\leq} \E\left[\sum_{t=1}^T \sum_{G' \in \mathcal{G}} \mathds{1}\left\{G'(t) = G', \exists (m,k) \in G': \abs{\hat{\mu}_{m,k}(t) - \mu_{m,k}} \geq \frac{\delta}{M}\right\}\right] + \sum_{G' \in \mathcal{G}} \E\left[\abs{\mathcal{E}_{2,G'}}\right], \\
         & \overset{(c)}{\leq} \E\left[\sum_{(m,k) \in \mathcal{M} \times \mathcal{K}} \abs{\mathcal{X}_{1,m,k}}\right] + \sum_{G' \in \mathcal{G}} \E\left[\abs{\mathcal{E}_{2,G'}}\right],                                                                                      \\
         & \overset{(d)}{\leq} \sum_{(m,k) \in \mathcal{M} \times \mathcal{K}} \E\left[\abs{\mathcal{X}_{1,m,k}}\right] + \sum_{G' \in \mathcal{G}} \E\left[\abs{\mathcal{E}_{2,G'}}\right],                                                                                      \\
         & \leq \sum_{( m,k) \in \mathcal{M} \times \mathcal{K}} \E\left[\abs{\mathcal{X}_{1,m,k}}\right] + \E\left[\abs{\mathcal{X}_{2,m,k}}\right],                                                                                                                             \\
         & \leq 4MK + \frac{2M^3 K}{\delta^2} + \frac{2MK \log T}{\Delta^{\mathrm{kl}}}.
    \end{align*}

    while inequalities (a) and (d) hold because of the linearity of expectation. Inequality (b) is derived from the nature of the algorithm, which only chooses one \(G' \in \mathcal{G}\) to be hinted at each round. Another important observation is that when the \(\ghint(t)\) is hinted at, there might be more than one edge \((m,k)\) satisfying \(\abs{\hat{\mu}_{m,k}(t) - \mu_{m,k}} \geq \frac{\delta}{M}\), which directly implies (c).

\end{proof}

\section{Proof of Theorem \ref{lemma:obs}}
\obsp*
\begin{proof}
    % [Proof of Theorem \ref{obs:msg}]
    The proof demonstrates that not knowing the magnitudes of other agents' statistics causes more conflicts in the future because it is  impossible for agents to distinguish between different possible reward models using their private information.

    Consider two instances, \(A\) and \(B\), with \(M=2\) and \(K=2\), where the sets of agents and arms are \(\mathcal{M} = \{m_1, m_2\}\) and \(\mathcal{K} = \{k_1, k_2\}\), respectively. We assume that the instances are:
    \begin{align*}
        A: & \bm{\mu}_{m_1} = \langle \frac{1}{2}+\epsilon,\frac{1}{2}-\epsilon\rangle \quad \bm{\mu}_{m_2} = \langle 2\epsilon,\epsilon\rangle \\
        B: & \bm{\mu}_{m_1} = \langle \frac{1}{2}+\epsilon,\frac{1}{2}-\epsilon\rangle \quad \bm{\mu}_{m_2} = \langle 4\epsilon,\epsilon\rangle \\
    \end{align*}
    Any fixed policies $\pi_{m_1}$ and $\pi_{m_2}$ chosen by the agents that achieve sub-linear regret on instance \(A\) will fail on instance \(B\).
    The reason is that in instance \(A\), the optimal solution for agent \(m_1\) is to pull arm \(k_1\).
    This means that every time a conflict occurs on arm \(k_1\), policy $\pi_{m_1}$ chooses to keep requesting to pull arm $k_1$ while policy $\pi_{m_2}$ leads agent \(m_2\) to stop requesting $k_1$ and try to pull $k_2$ instead.

    However, if the underlying instance is \(B\), agent \(m_1\)'s policy, $\pi_{m_1}$, still requests arm $k_1$ when a conflict on \(k_1\) occurs.
    The reason is that without knowing the magnitude of agent $m_2$'s utility, instances $A$ and $B$ are indistinguishable from $k_1$'s perspective. Therefore, agent \(m_1\) continues pulling arm \(k_1\) in instance \(B\) and $m_2$'s policy $\pi_{m_2}$ switches to request arm $k_2$ for the same reason. Meanwhile, the optimal matching under instance $B$ is for $m_1$ to pull $k_2$ and for $m_2$ to pull $k_1$, which leads policies $\pi_{m_1}$ and $\pi_{m_2}$ to a linear regret.
\end{proof}

\section{Rank Assignment.}\label{routine:rank}
The rank assignment algorithm from \citet{wang2020optimal} involves two steps: Orthogonalization, where players are assigned unique external ranks through blocks of length \( K+1 \), and Rank Assignment, where external ranks are converted to internal ranks using a modified Round-Robin scheme. More details can be found in \citet{wang2020optimal}.

\begin{lemma}[\citet{shi2021heterogeneous,wang2020optimal}, Lemma 1,2] \label{lemma:rank}
    The expected duration of the rank assignment procedure is less than $\frac{K^2 M}{K-M}+$ $2 K$ time steps. Once the procedure is completed, all players correctly learn the number of players $M$, each assigned with a unique index between 1 and $M$.
\end{lemma}

\section{Communication Protocol}
\label{routine:com}
% In decentralized multi-agent multi-armed bandit problems where direct communication is prohibited, agents can use collisions as signals to transmit bits, as proposed by \citet{boursier2019sic}. 
The implicit collision-based communication method introduced by \cite{boursier2019sic} involves the \emph{sender} signaling and transmitting statistics by pulling the \emph{receiver}'s \emph{communication arm} in a predetermined, time-dependent sequence governed by the unique ranks agents have been assigned through rank assignment phase. Specifically, during the communication phase preceding the exploration phase \(\rho\), the communication arm of agent \(m\) is \(k_m^{G^{\rho-1}}\). The sender \(m\) then either pulls the receiver \(m'\)'s communication arm (creating a collision; bit 1) or refrains from pulling it (no collision; bit 0). 

To minimize regret during communication, agents should be aware of the trade-off between keeping the communication brief, and the information loss that occurs when sending quantized format of the accurate decimal estimates \(\bm{\hat{\mu}}\) using a bit-by-bit integral method. The quantized format of \(\bm{\hat{\mu}}\) is denoted as \(\bm{\tilde{\mu}}\). The statistic \(\tilde{\mu}^\rho_{m,k}\) represents the shared information about the edge \((m,k)\) among all agents after communication, used for decision-making in the upcoming exploration epoch \(\rho\). To optimize communication duration while minimizing information loss, the \textit{Differential} communication protocol by \citet{shi2021heterogeneous} is employed. In this protocol, agents transmit \(\tilde{\delta}^\rho_{m,k} \coloneqq \tilde{\mu}^\rho_{m,k} - \tilde{\mu}^{\rho-1}_{m,k}\) using \(\left\lceil 1 + \frac{\log \rho}{2} \right\rceil\) bits. The receivers then update their statistics with \(\tilde{\mu}^\rho_{m,k} = \tilde{\mu}^{\rho-1}_{m,k} + \tilde{\delta}^\rho_{m,k}\). It has been shown that \[\hat{\mu}_{m,k}^\rho \leq \tilde{\mu}_{m,k}^\rho,\] and 
\begin{align}
    \tilde{\mu}^{\rho}_{m,k} - \hat{\mu}^{\rho}_{m,k} \leq \sqrt{\frac{1}{\rho}}, \label{ineq:comerr}
\end{align}
which are essential for the analysis. Additionally, transmitting \(\tilde{\delta}^\rho_{m,k}\) requires \(O(1)\) bits, and the communication duration per phase is \(O\left(M^2K\right)\) since each statistic \(\tilde{\mu}^\rho_{m,k}\) must be shared with \(M\) agents.

Unlike the leader-follower framework used in \citet{shi2021heterogeneous} and \citet{boursier2019sic}, where one agent (typically the one with rank 1) acts as a leader by gathering statistics from others and announcing the arm-pulling graph after running the $\HALG$ algorithm, this work assumes a peer-to-peer connection model. This approach allows for further analysis in situations where the shared statistics among agents may differ. Each agent independently runs the $\HALG$ algorithm using its own gathered statistics, which introduces additional complexity to the online decision-making problem. If agents reach different conclusions after running $\HALG$, it could lead to potential collisions and reward loss.

However, we assume that after all agents share their quantized statistics through the $\operatorname{Send2All}$ routine (Routine \ref{routin:sen2all}), they will have similar information and thus reach similar outcomes with $\HALG$. Under both $\HDETC$ and $\EBHDETC$, all agents simultaneously enter the communication phase by invoking the $\operatorname{Send2All}$ routine. This routine determines which agent will act as the receiver and execute the $\operatorname{Receive}$ subroutine (Routine \ref{routin:receive}), and which agent will act as the sender and execute the $\operatorname{Send}$ subroutine (Routine \ref{routin:send}) at each time step during the communication phase. In the $\operatorname{Send2All}$ routine, each agent maintains a counter $s$, which tracks the rank of the last sender, incrementing it by one after they send or receive statistics from other agents. An agent sends its statistics using the $\operatorname{Send}$ subroutine only if $s$ matches its rank; otherwise, it remains open to receive statistics being sent by the agent with rank $s$.

\begin{algorithm}[H]
    \floatname{algorithm}{Routine}
    \caption{ $\operatorname{Send2all}$: agent $m$}\label{routin:sen2all}

    \begin{algorithmic}[1]
        % \Input set of agents $\mathcal{M}$, set of arms  $\mathcal{K}$, time horizon $T$

       \Input quantized statistic $\tilde{\delta}_{m, k}^{\rho+1}$
        \State \textbf{Initialization:} $s \gets 1$,

        \While{$s\leq M$} \RightComment{number of the agents $M$ is known to the agents}
        \For{$m' \in [M]$}
        \If{$s = M$ and $m'\neq m$}
        
        \State $\operatorname{Send}(\tilde{\delta}_{m, k}^{\rho+1},m')$
        
        \ElsIf{$s \neq M$ and $m'\neq m$}
        \State $\operatorname{Receive}(\operatorname{bit\_string}(\tilde{\delta}_{m', k}^{\rho+1}),m')$

        \EndIf
        \EndFor
        \State $s\gets s+1$
        \EndWhile
    \end{algorithmic}
\end{algorithm}

\begin{algorithm}[H]
    \floatname{algorithm}{Sub-Routine}
    \caption{ $\operatorname{Send}$: agent $m$}\label{routin:send}

    \begin{algorithmic}[1]
        % \Input set of agents $\mathcal{M}$, set of arms  $\mathcal{K}$, time horizon $T$

       \Input quantized statistic $\tilde{\delta}_{m, k}^{\rho+1}$,receiver $m'$
        \State \textbf{Initialization:} $\bm{b} \gets \operatorname{bit\_string}(\tilde{\delta}_{m, k}^{\rho+1})$
        \State $l \gets \abs{\bm{b}}$
        \For{$i = \{1,2,3,\cdots,l\} $}
            \If{$\bm{b}[i] = 1$}
            \State pull the arm $k_{m'}^{G^\rho}$ \RightComment{epoch $\rho$ and matching $G^\rho$ are known to the agents}
            \Else
            \State pull the arm $k_{m}^{G^\rho}$
            \EndIf
            
        \EndFor
        
    \end{algorithmic}
\end{algorithm}
\begin{algorithm}[H]
    \floatname{algorithm}{Sub-Routine}
    \caption{ $\operatorname{Receive}$: agent $m$}\label{routin:receive}

    \begin{algorithmic}[1]
        % \Input set of agents $\mathcal{M}$, set of arms  $\mathcal{K}$, time horizon $T$

      \Input bit string $\bm{b'}$ with length $l'$, sender $m'$ 
        
        \State $\tilde{\delta}_{m', k}^{\rho+1} \gets 0$ \RightComment{epoch $\rho$ and arm index $k$ are known to the agents}
        \For{$i = \{1,2,3,\cdots,l'\} $}
        \State $\tilde{\delta}_{m', k}^{\rho+1} \gets 2\tilde{\delta}_{m', k}^{\rho+1} $
            \If{ collision }
            \State $\tilde{\delta}_{m', k}^{\rho+1} \gets \tilde{\delta}_{m', k}^{\rho+1} + 1$
            \EndIf
            
        \EndFor
    \end{algorithmic}
\end{algorithm}

\section{Analysis of $\HDETC$ and $\EBHDETC$ }\label{proof:sketch}
According to Lemma \ref{lemma:rank}, the regret caused by the rank assignment phase, which follows the same routine in both algorithms, has a time-independent upper bound as
\(
    R^{\pi_{\text{rank}}}(T) \leq M\left(\frac{K^2 M}{K-M} + 2 K\right).
\)

Furthermore, \citet{shi2021heterogeneous} demonstrated that the \textit{Differential} protocol requires \(O(1)\) bits for each peer-to-peer communication to transmit at each epoch. This protocol induces \(O(M^2K)\) regret at each communication phase, where \(M^2K\) is the total number of statistics shared among agents at each communication phase. However, unlike the rank assignment phase, which is common among all algorithms, different policies may have different criteria entering the exploitation phase, leading to varying number of communication phases and hint complexity. Therefore, we aim to analyze \(R^{\pi_{\text{exp}}}(T) + R^{\pi_{\text{com}}}(T)\) while maintaining a time-independent upper bound for \(R^{\pi_{\text{exp}}}(T)\) with the minimum possible number of hints $L^\pi(T)$.

\subsection{Proof of Theorem \ref{thm:hdetc}} \label{proof:hdetc}
\hdetcp*
\begin{proof}

    First, we argue that \(L^{\pi}(T) \in O\left(MT^\pi_0\right)\) because agents request \(M\) hints at each time step during the first \(T^{\pi}_0\) rounds. We then bound \(R^{\pi_{\text{com}}}(T)\). Under $\HDETC$, agents communicate every \(k\) rounds up to time step \(T^{\pi}_0\) that implies \(R^{\pi_{\text{com}}}(T) \in O\left(M^2T^{\pi}_0\right)\).

    The \(\HDETC\) algorithm has \(\frac{T^{\pi}_0}{K}\) decision-making rounds up to the time step \(T^{\pi}_0\), followed by one additional round for the period from \(T^{\pi}_0\) to \(T\). Thus, we can break down the exploration regret of \(\pi = \HDETC\) as follows:
    \[ R^{\pi_{\text{exp}}}(T) = R^{\pi_{\text{exp}}}(T^{\pi}_0) + R^{\pi_{\text{exp}}}(T^{\pi}_0:T) .\]

    We then apply the results of Lemma \ref{lemma:deba} to bound \(R^{\pi_{\text{exp}}}(T^{\pi}_0)\) as:
    \begin{align*}
        R^{\pi_{\text{exp}}}(T^{\pi}_0) \leq 19M^2K + 4M^2K^2 + \left(4M^3K + 8M^4K + 8M^4K^2\right)\delta^{-2}.
    \end{align*}

    To further bound \( R^{\pi_{\text{exp}}}(T^{\pi}_0:T) \), we write:
    \begin{align*}
        R^{\pi_{\text{exp}}}(T^{\pi}_0:T) & \leq 0T \Pr\left(G'^* = G^*\right) + MT \Pr\left(G'^* \neq G^*\right), \\
                                      & \leq MT \Pr\left(G'^* \neq G^*\right),                                 \\
                                      & \overset{(a)}{\leq} M,
    \end{align*}
    where the inequality (a) holds by Lemma \ref{lemma:main-hdetc}, which bounds the probability of \(G'^* \neq G^*\).

    Therefore, we can write:
    \begin{align}
        R^{\pi_{\text{exp}}}(T) \leq M+19M^2K + 4M^2K^2 \nonumber + \left(4M^3K + 8M^4K + 8M^4K^2\right)\delta^{-2}.
    \end{align}
    which results in a time-independent exploration regret.

\end{proof}

\subsection{Proof of Theorem \ref{thm:ebhdetc}}\label{proof:ebhdetc}
\ebhdetcp*
\begin{proof}
    For $\pi=\EBHDETC$ we denote the index of the last exploration epoch during which agents query for hints for \(K\) consecutive rounds, by \(\rho'\).  Finding an upper bound for $\rho'$, we can then argue that \(L^\pi(T) = MK\rho'\) and \(R^{\pi_{\text{com}}}(T) \in O\left(M^2K\rho'\right)\).

    Under \(\EBHDETC\), agents maintain a set of active edges \(\mathcal{C}^\rho\), which includes the edges that could potentially be part of \(G^*\). Consequently, agents decide to stop hinting at the beginning of epoch \(\rho'+1\) when \(\abs{\mathcal{C}^{\rho'+1}} = M\), such that \(\rho'\) and $T^\pi_0$ become random variables, as opposed to the fixed \(T^{\HDETC}_0\) in \(\HDETC\). We then, break down the exploration regret of \(\pi = \HDETC\) as follows:
    {\footnotesize
    \[ R^{\pi_{\text{exp}}}(T) = R^{\pi_{\text{exp}}}(T^\pi_0) + R^{\pi_{\text{exp}}}(T^\pi_0:T) ,\]
    }
    where $R^{\pi_{\text{exp}}}(T^\pi_0)$ follows the same bound as we proved in Lemma \ref{lemma:deba}. So, finding an upper bound for $R^{\pi_{\text{exp}}}(T^\pi_0:T)$, would induce an upper bound on the hint and communication complexity.

    Accordingly, we define the event $\mathcal{H}_T$ as
        {\footnotesize
            \[
                \mathcal{H}_{T}\coloneqq \mathds{1}\left\{\forall \rho \leq \frac{T}{K}, \forall G \in \mathcal{G}:\abs{U(G;\bm{\tilde{\mu}}^\rho) - U(G;\bm{\mu})} \leq 2M\epsilon^\rho\right\}.
            \]
        }
  In Lemma \ref{lemma:ht}, we prove that \(\Pr\left(\mathcal{H}_T \neq 1\right) \leq \frac{1}{T}\). Next, we focus on bounding the probability of \(G'^* = G^*\) to find an upper bound for \(R^{\pi_{\text{exp}}}(T^\pi_0:T)\). In Lemma \ref{lemma:matching}, we prove that the edges in the set of active edges form a perfect matching when \(\abs{\mathcal{C}^{\rho'}} = M\) the graph formed by the edges in \(\mathcal{C}^{\rho'}\) is exactly \(G^*\), which is the same outcome as \(\HALG\) finds with probability at least $1-\frac{M}{T}$.By Lemma \ref{lemma:finale}, we can prove that given \(\mathcal{H}_T = 1\), the index of the final epoch \(\rho' \leq \frac{64M^2\log\left(T\sqrt{2M}\right)}{\left(\Delta^\text{match}_{\min}\right)^2}\) and \(G'^* = G^*\). Thus, we write:
\begin{align*}
R^{\pi_{\text{exp}}}(T^\pi_0:T) &\leq MT\Pr\left(\mathcal{H}_T = 0\right) + \Pr\left(\mathcal{H}_T = 1\right)\left(0T\Pr\left(G'^* = G^*|\mathcal{H}_T = 1\right) + MT\Pr\left(G'^* \neq G^*|\mathcal{H}_T = 1\right)\right)\\
&\leq MT\Pr\left(\mathcal{H}_T = 0\right) +  MT\Pr\left(G'^* \neq G^*|\mathcal{H}_T = 1\right),\\
&\leq M + M^2.
\end{align*}
Finally, we can conclude that
\begin{align*}
    R^{\pi_{\text{exp}}}(T) &\leq M+M^2+19M^2K + 4M^2K^2 + \left(4M^3K + 8M^4K + 8M^4K^2\right)\delta^{-2}.
\end{align*}
We wrap up the proof by applying the result of Lemma \ref{lemma:finale} to bound \(L^\pi(T) \in O\left(\frac{M^3K\log\left(T\sqrt{2M}\right)}{\left(\Delta^\text{match}_{\min}\right)^2}\right)\) and \(R^{\pi_{\text{com}}}(T) \in O\left(\frac{M^4K\log\left(T\sqrt{2M}\right)}{\left(\Delta^\text{match}_{\min}\right)^2}\right)\).

\end{proof}

\subsection{Decentralized Event-Based Analysis}\label{sec:deba}
Here we present the event-based analysis for decentralized algorithms like what we did in Section \ref{eba}.
\begin{lemma}\label{lemma:deba}
    For policy $\pi \in \{\HDETC,\EBHDETC\}$, $R^{\pi_{\text{exp}}}(T)$ is bounded by
    \[R^{\pi_{\text{exp}}}(T)\leq19M^2K + 4M^2K^2 +\left(4M^3K + 8M^4K + 8M^4K^2\right)\delta^{-2}.\]
\end{lemma}
\begin{proof}
    Here, we use the same event-based analysis as in Section \ref{eba} to bound the newly defined sets of time steps like defined in Definition \ref{def:sets} as
    \begin{align*}
        \mathcal{A}' & \coloneqq \{t \geq 1: G(t) \neq G^*\},                                                                                                                             \\
        \mathcal{B}' & \coloneqq\{t \geq 1: \abs{U(G(t); \bm{\tilde{\mu}}(t)) - U(G(t); \bm{\mu})} \geq \delta\},                                                                         \\
        \mathcal{C}' & \coloneqq \{t \geq 1: U(G^*; \tilde{\bm{d}}(t)) < U\left(G^*; \bm{\mu}\right)\},                                                                                   \\
        \mathcal{D}' & \coloneqq \{t \geq 1: t \in \mathcal{A}' \setminus (\mathcal{B}' \cup \mathcal{C}'), \abs{U\left(G^*; \bm{\tilde{\mu}}(t)\right) - (G^*; \bm{\mu})} \geq \delta\}.
    \end{align*}

    These sets are analogous to those in Definition \ref{def:sets}, but they are defined on \(\bm{\tilde{\mu}}\) instead of \(\bm{\hat{\mu}}\). For any policy \(\pi \in \{\HDETC, \EBHDETC\}\), we have:
\[
R^{\pi_{\text{exp}}}(T) \leq M\E\left[\abs{\mathcal{A}'}\right].
\]
Thus, the focus is on finding an upper bound for \(\E\left[\abs{\mathcal{A}'}\right]\), which represents the number of time steps during which the central decision-maker selects a sub-optimal matching under each policy. Applying Lemma \ref{lemma:proof1}, we obtain:
\[
\E\left[\abs{\mathcal{A}'}\right] \leq \E\left[\abs{\mathcal{A}' \cup \mathcal{B}}\right] \leq \E\left[\abs{\mathcal{B}'}\right] + \E\left[\abs{\mathcal{C}'}\right] + \E\left[\abs{\mathcal{D}'}\right].
\]
Each of these can be bounded separately.

Thus, we can express the regret as:
\[
R^{\pi_{\text{exp}}}(T) \leq M\left(\E\left[\abs{\mathcal{B}'}\right] + \E\left[\abs{\mathcal{C}'}\right] + \E\left[\abs{\mathcal{D}'}\right]\right),
\]
which, by applying Lemmas \ref{lemma:proof22}, \ref{lemma:proof32}, and \ref{lemma:proof42} to bound \(\E\left[\abs{\mathcal{B}'}\right]\), \(\E\left[\abs{\mathcal{C}'}\right]\), and \(\E\left[\abs{\mathcal{D}'}\right]\) respectively, leads to:
\[
R^{\pi_{\text{exp}}}(T) \leq 19M^2K + 4M^2K^2 + \left(4M^3K + 8M^4K + 8M^4K^2\right)\delta^{-2}.
\]

\end{proof}

\subsection{Detailed Lemmas Required to Decentralized Event-Based Analysis}

Here, we prove the lemmas necessary to analyze the regret and hint complexity of $\HDETC$ and $\EBHDETC$ using the event-based analysis.
\begin{lemma}\label{lemma:proof22}

    For policy $\pi \in \{\HDETC,\EBHDETC\}$, the expected size of set $\mathcal{B'}$ is bounded by
    \[
        \E\left[\abs{\mathcal{B}'}\right]\leq 4MK^2 + \left(4M^2K + 8M^3K^2\right)\delta^{-2}\]
\end{lemma}

\begin{proof}
    Then we bound $\E\left[\abs{\mathcal{B}'}\right]$ by new sets of time steps defined as

    \[\mathcal{B}'_{m,k} \coloneqq \left\{t\geq 1:  (m,k) \in G(t), \abs{\tilde{\mu}_{m,k}(t) - \mu_{m,k}} \geq \frac{\delta}{M}\right\},\]

    \[\mathcal{B}'_{1,m,k} \coloneqq \left\{t \geq 1:  (m,k) \in G(t), \abs{\hat{\mu}_{m,k}(t) - \mu_{m,k}} \geq \frac{\delta}{2M}\right\},\]

    \[\mathcal{B}'_{2,m,k} \coloneqq \left\{t\geq 1:  (m,k) \in G(t), \tilde{\mu}_{m,k}(t) - \hat{\mu}_{m,k} \geq \frac{\delta}{2M}\right\}.\]
    Now we can write
    \begin{align*}
        \E\left[\abs{\mathcal{B}'}\right] \leq \sum_{(m,k) \in \mathcal{M}\times\mathcal{K}}\E\left[\abs{\mathcal{B}'_{m,k}}\right]\leq \sum_{(m,k) \in \mathcal{M}\times\mathcal{K}}\E\left[\abs{\mathcal{B}'_{1,m,k}}\right]+\E\left[\abs{\mathcal{B}'_{2,m,k}}\right].
    \end{align*}
    while the second inequality holds by triangle inequality. We then bound each $\E\left[\abs{\mathcal{B}'_{1,m,k}}\right]$ and $\E\left[\abs{\mathcal{B}'_{2,m,k}}\right]$ separately. We now write

    \begin{align}
        \E\left[\abs{\mathcal{B}'_{1,m,k}}\right] & =
        \E\left[\sum_t \mathds{1}\left\{(m,k)\in G(t),\abs{\hat{\mu}_{m,k}(t) - \mu_{m,k}} \geq \frac{\delta}{2M}\right\}\right],\nonumber                                                                        \\
                                                  & =\E\left[\sum_{\rho} \sum_{t \in \rho}  \mathds{1}\left\{(m,k)\in G^\rho,\abs{\hat{\mu}_{m,k}(t) - \mu_{m,k}} \geq \frac{\delta}{2M}\right\}\right],\nonumber \\
                                                  & = K\E\left[\sum_\rho \mathds{1}\left\{(m,k)\in G^\rho ,\abs{ \hat{\mu}_{m,k}(t) - \mu_{m,k}} \geq \frac{\delta}{2M}\right\}
        \right],\nonumber                                                                                                                                                                                         \\
                                                  & \leq 4K + 8M^2K\delta^{-2}.
    \end{align}

    where the last inequality holds by Lemma \ref{lemma:obs}.

    By setting $\frac{\delta}{2M} = \sqrt{\frac{1}{\rho}}$, we can conclude the for $\rho \geq 4M^2\delta^{-2}$ it always holds that $\tilde{\mu}^{\rho}_{m,k} - \hat{\mu}^{\rho}_{m,k} \leq \frac{\delta}{2M}$.
    Now, based on the fact that the algorithm asks for a hint at each time step before $t'$ and every epoch hints for all edges, we can write
    \begin{align*}
        \sum_{(m,k) \in \mathcal{M}\times\mathcal{K}}\E\left[\abs{\mathcal{B}'_{2,m,k}}\right] & \leq \sum_{\rho}\sum_{t\in\rho}\mathds{1}\left\{\tilde{\mu}_{m,k}(t) - \hat{\mu}_{m,k} \geq \frac{\delta}{2M}\right\}, \\
        & \leq 4M^2K\delta^{-2}.
    \end{align*}
    Now we can conclude that
    \[\E\left[\abs{\mathcal{B}'}\right]\leq 4MK^2 + \left(4M^2K + 8M^3K^2\right)\delta^{-2}.\]

\end{proof}
\begin{lemma}\label{lemma:proof32}
    For policy $\pi \in \{\HDETC,\EBHDETC\}$, the expected size of set $\mathcal{C'}$ is bounded by
    \[
        \E\left[\abs{\mathcal{C}'}\right] \leq 15MK.
    \]
\end{lemma}
\begin{proof}
    First, we argue that
    \[
        \Pr(U(G^*; \bm{\tilde{d}}(t)) < U(G^*; \bm{\mu})) \leq \Pr(U(G^*; \bm{d}(t)) < U(G^*; \bm{\mu})),
    \]
    where the inequality holds because \(\bm{\tilde{\mu}}(t) > \bm{\mu}(t)\) and by Lemma \ref{lemma:c_whhmab_kl1}.

    Thus, we conclude that
    \begin{align*}
        \E\left[\abs{\mathcal{C}'}\right] & \leq \E\left[\abs{\mathcal{C}}\right] \\
                                          & \leq 15M,
    \end{align*}
    where \(\mathcal{C}\) is defined in Definition \ref{def:sets}, and the last inequality is derived from Lemma \ref{lemma:proof3}. However, we may use a looser bound of \(\E\left[\abs{\mathcal{C}'}\right] \leq 15MK\). This is because each \( t \in \mathcal{C}' \) may occur at the beginning of an epoch where decision-making takes place. If the event at \( t \) coincides with the start of an epoch \(\rho\), it could result in \( K \) rounds of regret, as agents commit to \( G^\rho \) for the entire epoch.

\end{proof}

\begin{lemma}\label{lemma:proof42}
    For policy $\pi \in \{\HDETC,\EBHDETC\}$, the expected size of set $\mathcal{D'}$ is bounded by
    \[
        \E\left[\abs{\mathcal{D}'}\right] \leq 4MK + 8M^3K\delta^{-2}.
    \]
\end{lemma}
\begin{proof}

    We bound \(\E\left[\abs{\mathcal{D}'}\right]\) very similar to what we did for the set $\mathcal{B}'$. However, there is no need to account for the communication error here cause we addressed it while bounding \(\E\left[\abs{\mathcal{B}'}\right]\) for all the edges. Thus, by Lemma \ref{lemma:obs}, we can conclude \(\E\left[\abs{\mathcal{D}'}\right] \leq 4MK + 8M^3K\delta^{-2}.\)

\end{proof}

\subsection{Detailed Lemmas Required to  Proof  of Theorem  \ref{thm:hdetc}}
\begin{lemma}\label{lemma:main-hdetc}
    For policy $\pi=\HDETC$ and \( T^\pi_0 = \frac{9M^2 K \log(2MT)}{(\Delta^\text{match}_{\min})^2} \), the probability of \( G'^* = G^* \) is at least \( 1 - \frac{1}{T^2} \).

\end{lemma}
\begin{proof}
    Starting with defining $\epsilon^{\text{com}}_{m,k}(t) = \tilde{\mu}_{m,k}(t) - \hat{\mu}_{m,k}(t)$ and $\epsilon^{\text{exp}}_{m,k}(t) = \abs{\hat{\mu}_{m,k}(t) - \mu_{m,k}}$ we can write:

    \begin{align}
        \Pr\left(G'^* \neq G^* \right) & = \Pr\left( U(G'^*; \bm{\tilde{\mu}}(T^\pi_0)) > U(G^*; \bm{\tilde{\mu}}(T^\pi_0)) \right),\nonumber                                                                                                                                                                                                                                            \\
                                       & = \Pr\left(\sum_{(m',k') \in G'^*} \tilde{\mu}_{m',k'}(T^\pi_0) > \sum_{(m,k) \in G^*} \tilde{\mu}_{m,k}(T^\pi_0) \right),\nonumber                                                                                                                                                                                                             \\
                                       & \leq \Pr\left(\sum_{(m',k') \in G'^*} \hat{\mu}_{m',k'}(T^\pi_0) + \epsilon^{\text{com}}_{m',k'}(T^\pi_0) > \sum_{(m,k) \in G^*} \hat{\mu}_{m,k}(T^\pi_0) \right),\nonumber                                                                                                                                                                          \\
                                       & \leq \Pr\left(\sum_{(m',k') \in G'^*} \mu_{m',k'} + \epsilon^{\text{exp}}_{m',k'}(T^\pi_0) + \epsilon^{\text{com}}_{m',k'}(T^\pi_0) > \sum_{(m,k) \in G^*} \mu_{m,k} - \epsilon^{\text{exp}}_{m,k}(T^\pi_0) \right),\nonumber                                                                                                                        \\
                                       & \leq \Pr\left(\sum_{(m',k') \in G'^*} \epsilon^{\text{exp}}_{m',k'}(T^\pi_0) + \epsilon^{\text{com}}_{m',k'}(T^\pi_0) + \sum_{(m,k) \in G^*} \epsilon^{\text{exp}}_{m,k}(T^\pi_0) > \Delta^\text{match}_{\min} \right),\nonumber                                                                                                                         \\
                                       & \leq \sum_{(m',k') \in G'^*} \Pr\left( \epsilon^{\text{exp}}_{m',k'}(T^\pi_0) > \frac{\Delta^\text{match}_{\min}}{3M} \right) + \Pr\left( \epsilon^{\text{com}}_{m',k'}(T^\pi_0) > \frac{\Delta^\text{match}_{\min}}{3M} \right) + \sum_{(m,k) \in G^*} \Pr\left( \epsilon^{\text{exp}}_{m,k}(T^\pi_0) > \frac{\Delta^\text{match}_{\min}}{3M} \right),\nonumber \\
                                       & \leq \sum_{(m,k) \in G'^* \cup G^*} \Pr\left( \epsilon^{\text{exp}}_{m,k}(T^\pi_0) > \frac{\Delta^\text{match}_{\min}}{3M} \right) + \sum_{(m,k) \in G'^*} \Pr\left( \epsilon^{\text{com}}_{m,k}(T^\pi_0) > \frac{\Delta^\text{match}_{\min}}{3M} \right),\nonumber                                                                                     \\
                                       & \overset{(a)}{\leq} 2M e^{-\frac{2T^\pi_0 \delta^2}{9KM^2}} + \sum_{(m,k) \in G'^*} \Pr\left( \epsilon^{\text{com}}_{m,k}(T^\pi_0) > \frac{\Delta^\text{match}_{\min}}{3M} \right), \nonumber\\
                                       &\overset{(b)}\leq 2M e^{-\frac{2T^\pi_0 \delta^2}{9KM^2}},\nonumber\\
                                       &\overset{(c)}\leq \frac{1}{T^2},\nonumber
    \end{align}
    where (a) holds by applying Hoeffding's inequality, under the assumption that all edges have been hinted at least \(\frac{T^\pi_0}{K}\) times by time step \(T^\pi_0\), and (b) holds because inequality \eqref{ineq:comerr} implies that after at least \(\frac{T^\pi_0}{K}\) observations for an edge \((m,k)\), the condition \(\epsilon^{\text{exp}}_{m,k}(T^\pi_0) > \frac{\Delta^\text{match}_{\min}}{3M}\) occurs with zero probability, resulting in 
\(
\sum_{(m,k) \in G'^*} \Pr\left( \epsilon^{\text{com}}_{m,k}(T^\pi_0) > \frac{\Delta^\text{match}_{\min}}{3M} \right) = 0. 
\). Finally, we replace $T_0^\pi$ with its actual value to achieve (c).

\end{proof}
\subsection{Detailed Lemmas Required to  Proof  of Theorem  \ref{thm:ebhdetc}}
\begin{lemma}\label{lemma:ht}
    Under policy $\pi = \EBHDETC$, we have \(\Pr\left(\mathcal{H}_{T} = 1\right) \geq 1 - \frac{1}{T}\).
    
\end{lemma}
\begin{proof}
    For a fixed $\rho < \frac{T}{K}$, and any fixed edge $(m,k) \in \mathcal{M}\times\mathcal{K}$, we can observe that $N^\rho_{m,k} \geq \rho$ because $\EBHDETC$ hints for that edge during each exploration epoch exactly once. Now for $(m,k)$ we can write
    \begin{align}
        \Pr\left(\abs{\tilde{\mu}_{m,k}^\rho-\mu_{m,k}} > 2\epsilon^\rho\right) & \leq \Pr\left(\tilde{\mu}_{m,k}^\rho-\hat{\mu}_{m,k}^\rho > \epsilon^\rho\right)+\Pr\left(\abs{\hat{\mu}_{m,k}^\rho-\mu_{m,k}} > \epsilon^\rho\right),\nonumber \\
                                                                                & \overset{(a)}\leq \Pr\left(\abs{\hat{\mu}_{m,k}^\rho-\mu_{m,k}} > \epsilon^\rho\right),\nonumber                                                                \\
                                                                                & \overset{(b)}\leq 2e^{\log\left(\frac{\eta^2}{4}\right)},\nonumber                                                                                              \\ \label{ineq:errprob}
                                                                                & \leq \frac{\eta^2}{2},
    \end{align}
    where (a) follows from the communication error bound \(\eqref{ineq:comerr}\), which implies \(\Pr\left(\tilde{\mu}_{m,k}^\rho - \hat{\mu}_{m,k}^\rho > \epsilon^\rho\right) = 0\), and (b) is derived using Hoeffding's inequality assuming \(N^\rho_{m,k} \geq \rho\).

    Since $\EBHDETC$ hints for all edges until the end of epoch $\rho$, inequality \eqref{ineq:errprob} holds for all \((m, k) \in \mathcal{M} \times \mathcal{K}\). Therefore, by summing the error bound over all \(M\) edges included in any matching \(G \in \mathcal{G}\) and then taking the union bound over all the edges, we can write:
    \begin{align*}
        \Pr\left(\exists G \in \mathcal{G}: \abs{U(G,\bm{\tilde{\mu}}^\rho) - U(G,\bm{\mu})} > 2M\epsilon^\rho\right) & \leq \Pr\left(\exists (m,k) \in \mathcal{M}\times\mathcal{K}:\abs{\tilde{\mu}_{m,k}^\rho-\mu_{m,k}} > 2\epsilon^\rho \right), \\
                                                                                                                      & \leq \frac{MK\eta^2}{2}.
    \end{align*}

    Finally, we take union bound on all $\rho \leq \frac{T}{K}$ and write:
    \begin{align*}
        \Pr\left(\mathcal{H}_T \neq 1\right) & \leq \Pr\left(\exists \rho\leq \frac{T}{K}: \exists G \in \mathcal{G}: \abs{U(G;\bm{\tilde{\mu}}^\rho) - U(G;\bm{\mu})} > 2M\epsilon^\rho \right), \\
                                             & \leq \frac{TM\eta^2}{2},                                                                                                                           \\
                                             & \overset{(a)}\leq \frac{1}{T},
    \end{align*}
    where (a) holds by setting $\eta = \sqrt{\frac{2}{MT^2}}$.
\end{proof}
\begin{lemma}\label{lemma:matching}
    For any \(\rho' \leq \frac{T}{K}\), under the policy \(\pi = \EBHDETC\), if \(\lvert \mathcal{C}^{\rho'} \rvert = M\), the edges in \(\mathcal{C}^{\rho'}\) form a matching that, with probability at least \(1 - \frac{M}{T}\), corresponds to \(G^*\).

.

\end{lemma}
\begin{proof}
    The algorithm $\EBHDETC$ stops hinting and communicating at the beginning of epoch \(\rho'\) if \(|\mathcal{C}^{\rho'}| = M\). Suppose the edges \((m, k) \in \mathcal{C}^{\rho'}\) do not form a matching. Consequently, there must exist an agent \(m'\) such that \(\nexists k \in \mathcal{K}\) where \((m', k) \in \mathcal{C}^{\rho'}\), which is impossible. This is because, according to how $\EBHDETC$ updates \(\mathcal{C}^\rho\) at each epoch, any edge \((m, k) \in \mathcal{C}^{\rho'}\) is included in a matching \(G^{\rho'}_{(m, k)}\) such that \(U(G^{\rho'}; \bm{\tilde{\mu}}^{\rho'}) - U(G^{\rho'}_{(m, k)}; \bm{\tilde{\mu}}^{\rho'}) < 4M \bm{\epsilon}^{\rho'}\), where \(m'\) is connected to \(k_{m'}^{G^{\rho'}_{(m, k)}}\), implying that \((m', k_{m'}^{G^{\rho'}_{(m, k)}})\) should be in \(\mathcal{C}^{\rho'}\).

    We now show that if agents stop hinting and communicating when $\abs{\mathcal{C}^{\rho'}} = M$, the matching formed by the edges in $\mathcal{C}^{\rho'}$ is $G^*$ with a probability of at least $1 - \frac{M}{T}$. To support this claim, we bound the probability of an edge $(m,k)$ being eliminated by $\EBHDETC$ at any epoch $\rho$. Assume that $(m,k) \in G^*$ is eliminated from $\mathcal{C}^\rho$, i.e., $U(G^\rho;\bm{\tilde{\mu}}^\rho) - U(G_{(m,k)}^\rho;\bm{\tilde{\mu}}^\rho) > 4M\epsilon^\rho$. We first define the following events $A_{m,k}$ and $B_{m,k}$ for an edge $(m,k) \in G^*$ as:

    \begin{align*}
        A_{m,k} & \coloneqq \left\{\exists \rho \leq \rho': U(G^\rho;\bm{\tilde{\mu}}^\rho) - U(G_{(m,k)}^\rho;\bm{\tilde{\mu}}^\rho) > 4M\epsilon^\rho\right\}, \\
        B_{m,k} & \coloneqq \left\{\exists \rho \leq \rho': U(G^\rho;\bm{\tilde{\mu}}^\rho) - U(G^*;\bm{\tilde{\mu}}^\rho) > 4M\epsilon^\rho\right\}
    \end{align*}

    We can then write:

    \begin{align*}
        \Pr\left(A_{m,k}\right) & \overset{(a)}\leq \Pr\left(B_{m,k}\right),                                                                                    \\
                                & \leq \Pr\left(\mathcal{H}_T \neq 1\right) + \Pr\left(\mathcal{H}_T = 1\right) \Pr\left(B_{m,k} \mid \mathcal{H}_T = 1\right), \\
                                & \overset{(b)}\leq \Pr\left(\mathcal{H}_T \neq 1\right),
    \end{align*}

    where (a) holds because $G^*$ is a matching containing $(m,k)$ while $G^\rho_{(m,k)}$ is one of such matchings with the highest utility, i.e., $U(G_{(m,k)}^\rho;\bm{\tilde{\mu}}^\rho) \geq U(G^*;\bm{\tilde{\mu}}^\rho)$. Inequality (b) follows from the observation that, given $\mathcal{H}_T = 1$, the following inequalities hold:

    \begin{align*}
        U(G^\rho;\bm{\tilde{\mu}}^\rho) - 2M\epsilon^\rho & \leq U(G^\rho;\bm{\mu}), \\
        U(G^*;\bm{\tilde{\mu}}^\rho) + 2M\epsilon^\rho    & \geq U(G^*;\bm{\mu}),
    \end{align*}

    which implies:

    \[ U(G^\rho;\bm{\tilde{\mu}}^\rho) - U(G^*;\bm{\tilde{\mu}}^\rho) \leq 4M\epsilon^\rho - \Delta^\text{match}_{\min}, \]

    making $\Pr\left(B_{m,k} \mid \mathcal{H}_T = 1\right) = 0$. We demonstrated that any edge \((m,k) \in G^*\) is eliminated from the set of active edges with a probability of at most \(\frac{1}{T}\). By applying the union bound to all \(M\) edges of the matching \(G^*\), we can show that the final \(M\) edges appearing in \(\mathcal{C}^{\rho'}\) form \(G^*\) with a probability of at least \(1-\frac{M}{T}\).

\end{proof}
\begin{lemma}\label{lemma:finale}
    Assuming a unique optimal matching \(G^* \in \mathcal{G}\), under the policy \(\pi = \EBHDETC\), the algorithm stops querying for hints by epoch \(\rho' \leq \frac{64M^2\log\left(T\sqrt{2M}\right)}{\left(\Delta^\text{match}_{\min}\right)^2}\) with probability at least \(1 - \frac{1}{T}\).

\end{lemma}
\begin{proof}
    We first define the event \(A_T\) as
    \[
        A_T \coloneqq \mathds{1}\left\{\exists \rho' \leq \frac{T}{K}: \abs{\mathcal{C}^{\rho'}} = M\right\},
    \]
    which represents whether \(\EBHDETC\) stops hinting and enters the exploitation phase at the beginning of some epoch \(\rho'\).

    Now we can write
    \begin{align}
        \Pr\left(A_T \neq 1\right) & = \Pr\left(\mathcal{H}_T = 1\right)\Pr\left(A_T \neq 1|\mathcal{H}_T = 1\right) + \Pr\left(\mathcal{H}_T = 0\right)\Pr\left(A_T \neq 1|\mathcal{H}_T = 0\right), \nonumber \\
                                   & \leq \Pr\left(\mathcal{H}_T = 1\right)\Pr\left(A_T \neq 1|\mathcal{H}_T = 1\right) + \Pr\left(\mathcal{H}_T = 0\right). \label{ineq:matching}
    \end{align}

    From Lemma \ref{lemma:ht}, we know that \(\Pr\left(\mathcal{H}_T = 0\right) \leq \frac{1}{T}\). Thus, we prove that, given \(\mathcal{H}_T = 1\), there exists an epoch \(\rho' \leq \frac{T}{K}\) such that \(\abs{\mathcal{C}^{\rho'}} = M\), that makes \(\Pr\left(\mathcal{H}_T = 1\right)\Pr\left(A_T \neq 1|\mathcal{H}_T = 1\right) = 0\).

    Accordingly, given $\mathcal{H}_T = 1$, we bound the index of the last epoch \(\rho\) where an edge \((m',k') \in \mathcal{M} \times \mathcal{K}\) such that \((m',k') \notin G^*\) can be in \(\mathcal{C}^{\rho}\). As a consequence of Lemma \ref{lemma:matching}, we know that all the edges $(m,k) \in G^*$ will remain in $\mathcal{C}$ till time step $T$ when $\mathcal{H}_T = 1$. We also know that $ U(G^*;\bm{\tilde{\mu}}^\rho) \leq U(G^\rho;\bm{\tilde{\mu}}^\rho)$. Thus, if $4M\epsilon^\rho < \frac{\Delta^\text{match}_{\min}}{2}$ we can write that
    \begin{align*}
        U(G^*;\bm{\tilde{\mu}}^\rho) - U(G^\rho_{(m',k')};\bm{\tilde{\mu}}^\rho) & \geq  U(G^*;\bm{\mu}) - U(G^\rho_{(m',k')};\bm{\mu})-4M\epsilon^\rho, \\
                                                                                 & \geq \Delta^\text{match}_{\min} - 4M\epsilon^\rho,                        \\
                                                                                 & \geq 4M\epsilon^\rho.
    \end{align*}
    Thus, all the edges $(m',k')\notin G^*$ can not remain in $\mathcal{C}^\rho$ for $\rho \geq \frac{64M^2\log\left(T\sqrt{2M}\right)}{\left(\Delta^\text{match}_{\min}\right)^2}$ given $\mathcal{H}_T = 1$.

    Therefore, re-write inequality \eqref{ineq:matching} as
    \begin{align*}
        Pr\left(A_T \neq 1\right)
         & \leq \Pr\left(\mathcal{H}_T = 1\right)\Pr\left(A_T \neq 1|\mathcal{H}_T = 1\right) + \Pr\left(\mathcal{H}_T = 0\right), \\
         & \leq \Pr\left(\mathcal{H}_T = 0\right),                                                                                 \\
         & \leq \frac{1}{T}.
    \end{align*}

\end{proof}

\section{Extended Experiments}\label{sec:experiment}
Here, we present additional simulation plots that provide further insights, complementing the main plots shown in Figure \ref{fig:figss}.

One key aspect not fully illustrated in Figure \ref{fig:figss} is the distinct difference in hint complexity between $\GHCLA$ and $\GPHCLA$. As demonstrated in Theorem \ref{lemma:hcla2}, the hint complexity of $\GHCLA$ is up to \(M\) times greater than that of $\GPHCLA$ in the worst-case scenario. Although constructing a worst-case instance for small graphs is challenging, Figure \ref{fig1:figss} offers a direct comparison of the performance of $\GPHCLA$ and $\GHCLA$ for an instance with $M=2$, $K=2$, and $\Delta^\text{match}_{\min} = 0.45$, averaged over 50 trials with $T=10^5$.

\begin{figure}[H]
        \centering
        \begin{subfigure}[b]{0.99\textwidth}
            \centering\includegraphics[width=0.5\textwidth]{legend.png}
            % \caption[Network2]%
            % {{\small Network 1}}    
            % \label{fig:mean and std of net14}
        \end{subfigure}
        \\
        \setcounter{subfigure}{0}
        \begin{subfigure}[b]{0.23\textwidth}
            \centering
            \includegraphics[width=\textwidth]{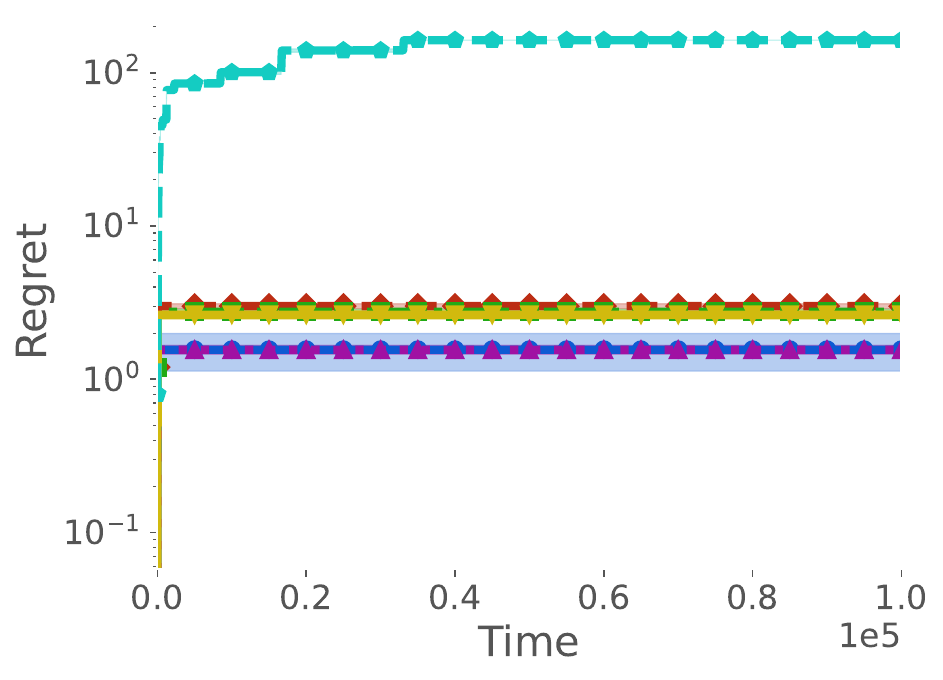}
            \caption[Regret]%
            {{\small Exploration Regret}}      
            \label{fig1:exp}
        \end{subfigure}
        \hfill
        \begin{subfigure}[b]{0.23\textwidth}  
            \centering 
            \includegraphics[width=\textwidth]{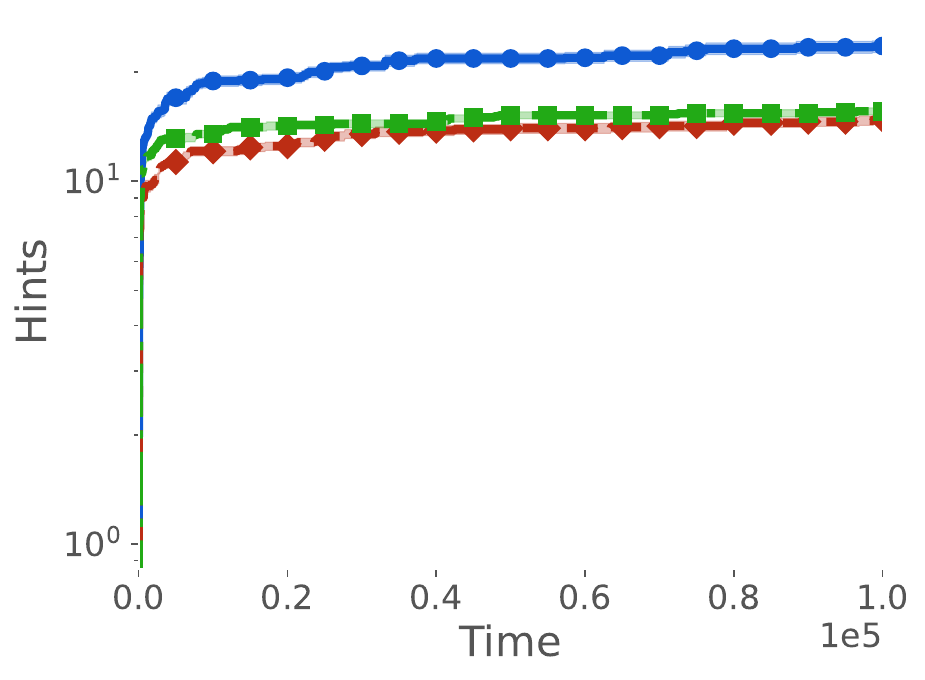}
            \caption[]%
            {{\small Centralized Hint Complexity}}    
            \label{fig1:chints}
        \end{subfigure}
        \hfill
        \begin{subfigure}[b]{0.23\textwidth}   
            \centering 
            \includegraphics[width=\textwidth]{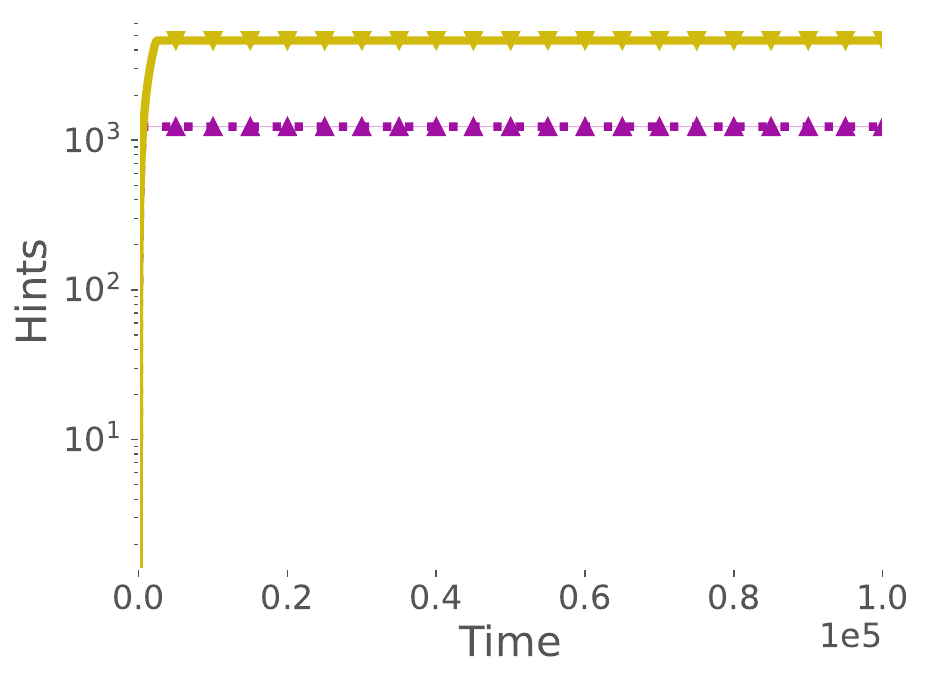}
            \caption[]%
            {{\small Decentral. Hint Complexity}}    
            \label{fig1:dhints}
        \end{subfigure}
        \hfill
        \begin{subfigure}[b]{0.23\textwidth}   
            \centering 
            \includegraphics[width=\textwidth]{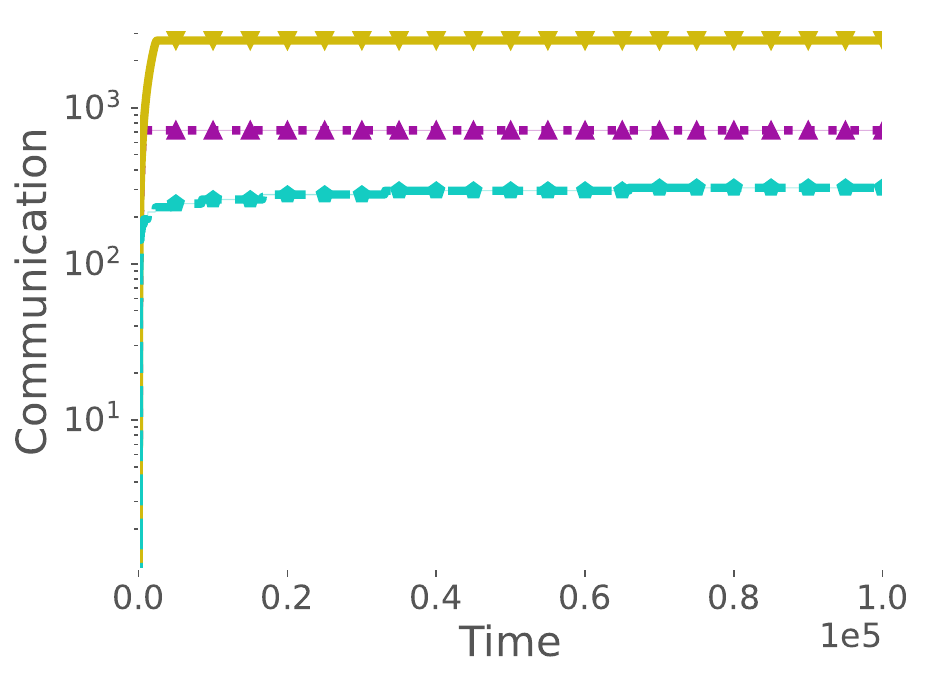}
            \caption[]%
            {{\small Communication Regret}}    
            \label{fig1:dcom}
        \end{subfigure}
        \caption[ ]
        {{\small Figure \ref{fig1:chints}
 demonstrates that $\GPHCLA$ outperforms $\GHCLA$ in terms of $L^\pi(T)$.}}
        \label{fig1:figss}
    \end{figure}

In the following plots in Figure \ref{fig2:figss}, we present the results of simulations conducted on a larger graph with $M=4$, $K=7$, and $\Delta^\text{match}_{\min} = 0.20$, averaged over 50 trials with $T=10^5$. These plots illustrate how the slight increase in the size of the instance, compared to the instance in Figure \ref{fig:figss}, amplifies the performance gap between $\HCLA$ and the other algorithms.

\begin{figure}[H]
        \centering

        \setcounter{subfigure}{0}
        \begin{subfigure}[b]{0.23\textwidth}
            \centering
            \includegraphics[width=\textwidth]{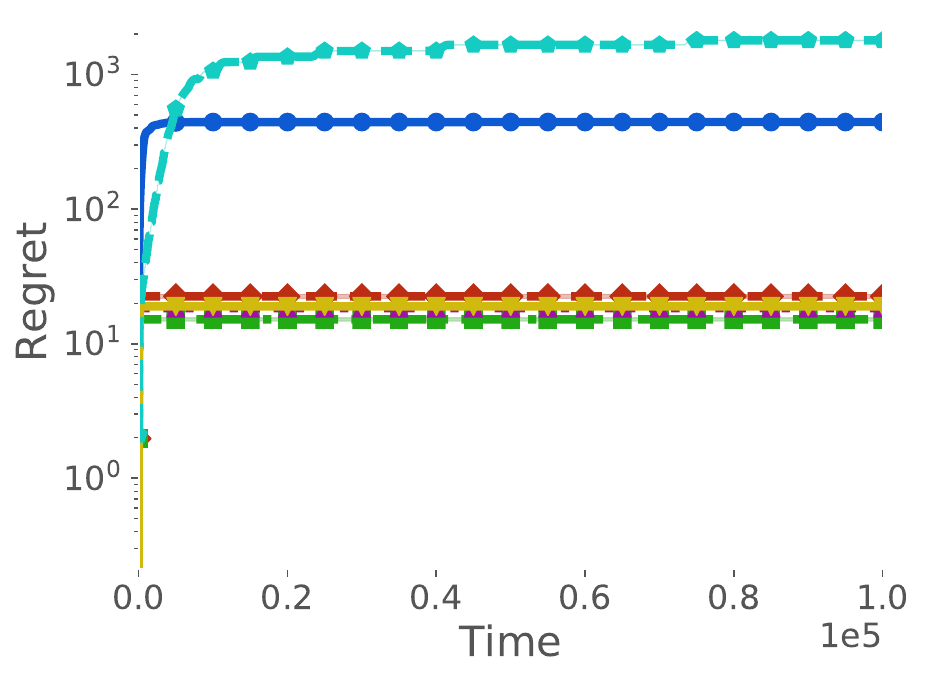}
            \caption[Regret]%
            {{\small Exploration Regret}}      
            \label{fig2:exp}
        \end{subfigure}
        \hfill
        \begin{subfigure}[b]{0.23\textwidth}  
            \centering 
            \includegraphics[width=\textwidth]{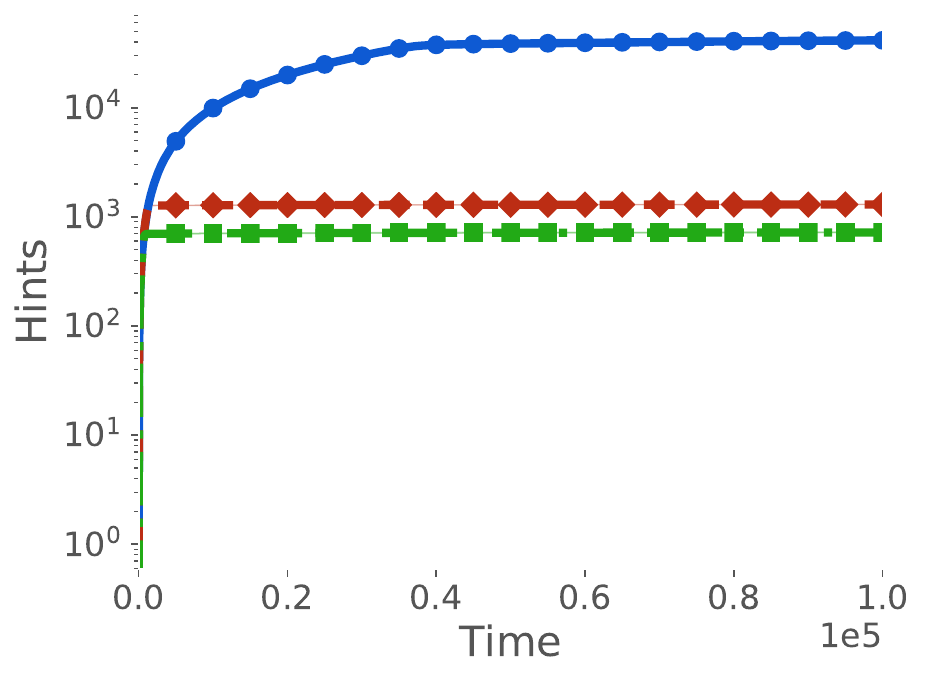}
            \caption[]%
            {{\small Centralized Hint Complexity}}    
            \label{fig2:chints}
        \end{subfigure}
        \hfill
        \begin{subfigure}[b]{0.23\textwidth}   
            \centering 
            \includegraphics[width=\textwidth]{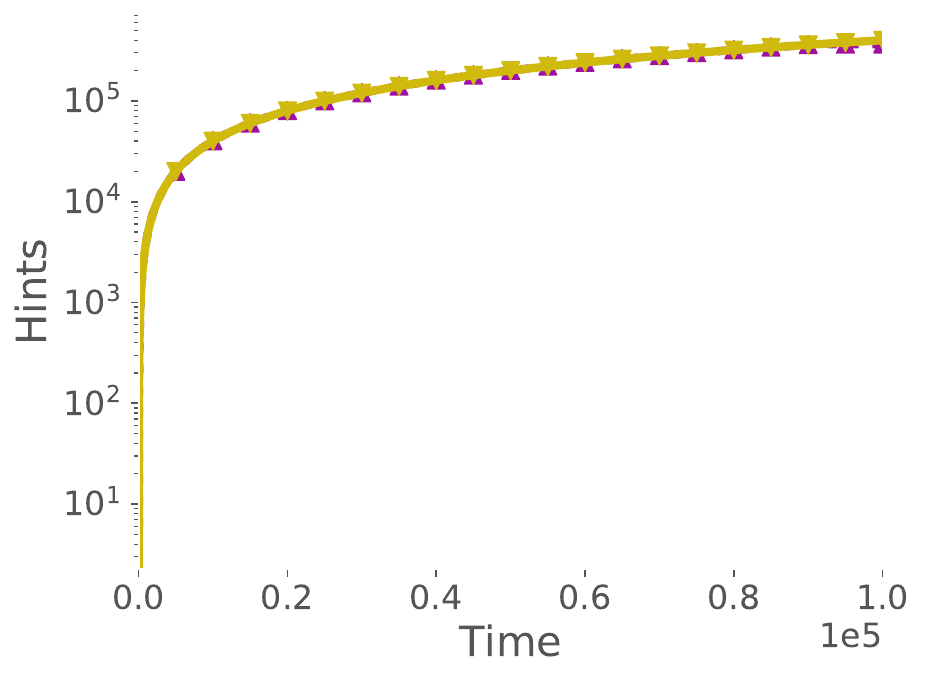}
            \caption[]%
            {{\small Decentral. Hint Complexity}}    
            \label{fig2:dhints}
        \end{subfigure}
        \hfill
        \begin{subfigure}[b]{0.23\textwidth}   
            \centering 
            \includegraphics[width=\textwidth]{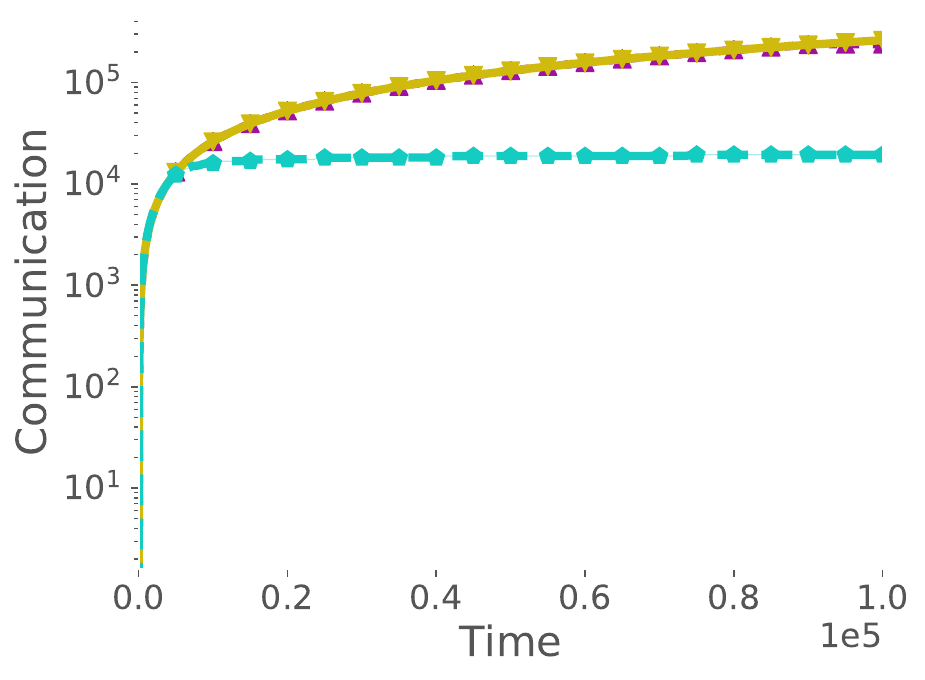}
            \caption[]%
            {{\small Communication Regret}}    
            \label{fig2:dcom}
        \end{subfigure}
        \caption[ ]
        {{\small Figures \ref{fig2:chints} and \ref{fig2:exp} illustrate the inefficiency of $\HCLA$ when subjected to a slight increase in the size of the instance.}}
        \label{fig2:figss}
    \end{figure}

\section{Optimality of the Results}
\label{App:OofR}
\subsection{A Lower Bound on the Necessary Number of Hints}
An important concern in designing efficient learning policies for $\HHMMAB$, is the number of queried hints $L^\pi(T)$. In Theorem \ref{thm:imphi}, we prove that all uniformly fast convergent learning policies, defined in Definition \ref{def:ufc}, require $\Omega(\log T)$ hints to achieve time-independent regret.  Accordingly, we call a policy $\pi$,  \textit{asymptotically hint optimal} if $L^\pi(T) \in O(\log T)$.
\begin{definition}\label{def:ufc}
    A policy $\pi$ is called uniformly fast convergent if for all $0<\alpha<1$ and all sub-optimal matchings $G \neq G^*$, $$\E\left[N^{\pi}_{G}(T)\right]=o\left(T^\alpha\right),$$
    holds.
\end{definition}
\begin{theorem}\label{thm:imphi}\label{thm:hint-lower-bound}
    Any uniformly fast convergent policy $\pi$ with time-independent regret should query for $L^\pi(T) \in \Omega(\log T)$ hints.
\end{theorem}
\begin{proof}
    % [Proof of Theorem \ref{thm:imphi}]
    Recall that in classic multi-armed bandits \citet{garivier2019explore,lai1985asymptotically}, for any uniformly fast convergent policy $\pi$, we have

    $$
        \liminf _{T \rightarrow \infty} \frac{\E\left[N_{G}^{\pi}\left(T\right)\right]}{\log T} \geqslant \frac{1}{\mathrm{kl}\left(U(G;\bm{\mu}), U(G^*;\bm{\mu})\right)}, \forall G \neq G^*,
    $$

    If the number of hints is $o(\log T)$, then there must be at least $\Omega(\log T)$ non-hint observations on some suboptimal arms. Consequently, for any uniformly fast convergent policy, the regret is also $\Omega(\log T)$, which is time-dependent.

\end{proof}

This theorem indicates that the upper bounds we have achieved for $L^\pi(T)$ are asymptotically optimal for $\pi \in {\HCLA,\GPHCLA,\GHCLA,\HDETC,\EBHDETC}$. Additionally, we can observe that the upper bound $O\left(MK\log T\right)$ obtained for $L^\GPHCLA(T)$ is also tight. This is because there are $MK$ edges, and there always exist instances where each edge must be explored or hinted at least $O\left(\log T\right)$ times.

\subsection{Bounds on the Necessary Regret}

When estimations are accurate for all agent pairs \(\left(m_{1}, m_2\right)\) with potential matches \(\left(k_{1}, k_2\right)\), the Hungarian Algorithm finds the optimal matching with zero regret. This leads to an upper bound of \(O(M^3 K^2)\), as there are \(O(M^2 K^2)\) exploration rounds, each incurring a maximum regret of \(M\). Our bound is tighter, particularly as \(K\) increases, due to the assumption \(M \leq K\).  

For the lower bound, in the homogeneous setting, agents rank arms globally to identify the top \(M\) arms, requiring sufficient exploration to distinguish them from the remaining \(K-M\) arms. This results in an \(O(M^2 K)\) bound, which also serves as a lower bound in our case. However, there remains a gap of \(M^2\) between our centralized bound and this lower bound, which remains an open question.  

In the decentralized case, our upper bound remains \(O(M^3 K^2)\), aligning with the centralized case. However, proving a matching lower bound is an open direction for future work.

\subsection{A Lower Bound on the Necessary Number of Communications}
Minimizing communication phases is a crucial challenge in designing algorithms for decentralized setups. Even in offline scenarios without uncertainty, bounding the required communication remains difficult. While time-independent regret might seem achievable with time-independent communication phases in homogeneous multi-agent multi-armed bandits, we conjecture that achieving time-independent regret in $\HHMMAB$s requires \(O\left(\log T\right)\) communication epochs, similar to \(L^\pi(T)\).

\begin{conjecture}\label{Conj:conj}
Any uniformly fast-converging policy \(\pi\) with time-independent regret must communicate \(\Omega\left(\log T\right)\) times, i.e., \(R^{\pi_{\text{com}}} \in \Omega\left(\log T\right)\).
\end{conjecture}

The intuition behind the proof of Conjecture \ref{Conj:conj} involves dividing the time horizon into exploration intervals separated by communication phases. Policies like $\HDETC$ and $\EBHDETC$ achieve time-independent regret by uniformly querying hints until \( T^\pi_0 \in O(\log T) \). If the length of any exploration interval before \( T^\pi_0 \) depends on \( t \), there exist instances where any learning policy \(\pi\) incurs time-dependent regret. Conversely, if all exploration intervals before \( T^\pi_0 \) have time-independent lengths, their number must be \( O(\log T) \), as they span the first \( T^\pi_0 \) steps. This number coincides with the required communication phases. Assessing the probability of time-dependent regret when interval lengths depend on \( T \) remains an open and non-trivial question for future work.

\section{Structure of the Optimal Matching}
We argue that any optimal matching on a weighted bipartite graph implies a hierarchical structure on agents, referred to as the \textbf{Multi-Level Agent Structure (MLAS)}, based on the position of their optimal match in their sorted weight list. This hierarchical structure is central to designing efficient decentralized learning algorithms using the described hint inquiry mechanisms and the associated regret analysis.

\begin{definition}\label{def:mlas}
    In a given perfect matching $G$ over $\mathcal{M}$ and $\mathcal{K}$ and a sorted list $\bm{\mu}^s_m$ for each agent $m \in \mathcal{M}$, we say $G$ follows $\MLAS$ if all the following existence conditions hold simultaneously:
    \[
        \exists m_1 \in \mathcal{M} :\quad p_{m_1}(S) \in \{1\},
    \]
    \[
        \exists m_2 \in \mathcal{M} \setminus \{m_1\} :\quad p_{m_2}(S) \in \{1, 2\},
    \]
    \[
        \vdots\;\;\;\;\;\;\;\;\;\;\;\;\;\;\;\;\;\;\;\;\;\;\;\;\;\vdots
    \]
    \[
        \exists m_M \in \mathcal{M} \setminus \{m_1, m_2, \ldots, m_{M-1}\} :\quad p_{m_M}(S) \in \{1, 2, \ldots, K\},
    \]
    where $p_{m}(S)$ denotes the index of $k^G_m$ in $\bm{\mu}^s_m$.

\end{definition}
Intuitively, Definition \ref{def:mlas} indicates that if $\MLAS$ holds for $S$, then there is an agent who is matched with his most preferred arm under $\bm{\mu}^s$s, another agent who is matched with his most or second most preferred arm, and so forth. Lemma \ref{lemma:mlas} demonstrates that an optimal matching $G^*$ adheres to $\MLAS$ when considering the $\bm{\mu}^s_m$s, which are the sorted versions of the underlying vectors $\bm{\mu}_m$. It is important to note that sorting the underlying $\bm{\mu}_m$s does not alter $G^*$; thus, applying $\HALG$ to $\bm{\mu}^s_m$s will also produce $G^*$ as the optimal matching.

\begin{lemma}\label{lemma:mlas}
    Lets denote the sorted lists of the underlying edge weights by $\bm{\mu}^s_m$. Then, the optimal matching $G^*$ over $\bm{\mu}^s_m$s follows $\MLAS$ structure.
\end{lemma}
\begin{proof}[Proof of Lemma \ref{lemma:mlas}]
    We first prove that there exists an agent $m_1$ who satisfies the topmost existence condition by having $p_{m_1}(G^*) = 1$, meaning that $k_{m_1}^{G^*}$ is the most preferred arm by agent $m_1$. We then remove $k_{m_1}^{G^*}$ from the sorted lists $\bm{\mu}^s_{m' \neq m_1}$s and use induction to show the existence of another agent $m_2$ who satisfies the second condition. We repeat this process to prove the conditions for the remaining agents.

    To start, we prove that given an optimal matching $G^*$, there exists an agent $m_1$ such that $p_{m_1}(G^*) = 1$. We achieve this by constructing a directed preference graph $PG_{G^*}$ where the nodes represent matched pairs $( m, k^{G^*}_m )$. An edge is directed from node $( m, k^{G^*}_m)$ to $( m', k^{G^*}_{m'})$ if $\mu_{m, k^{G^*}_{m'}} > \mu_{m, k^{G^*}_m}$.

    We argue that $PG_{G^*}$ must be an acyclic graph; otherwise, it would contradict the optimality of $G^*$. Suppose $PG_{G^*}$ contains a cycle of length $i$, as depicted in Figure \ref{fig:cycle}. We label the nodes as $( m_1, k^*_1 ), ( m_2, k^*_2 ), ( m_3, k^*_3 ), \ldots, ( m_i, k^*_i )$.

    \begin{figure}[!htb]
        \centering
        \scalebox{.75}{
            \begin{tikzpicture}[node distance={35mm}, main/.style = {draw, circle},scale=0.5]
                \node[main] (1) {$( m_1,k^*_1 )$};
                \node[main] (2) [right of=1] {$( m_2,k^*_2 )$};
                \node[main] (3) [right of=2]{$( m_{3},k^*_{3} ) $};
                \node[main] (4) [right of=3]{$( m_i,k^*_i )$};
                \draw[->] (4) to [out=75, in=75, looseness=0.25] (1);
                \draw[->] (1) -- (2);
                \draw[->] (2) -- (3);
                \draw[dashed ,->] (3) -- (4);
            \end{tikzpicture}
        }
        \caption{An $i$-cycle contained in $PG_{G^*}$}
        \label{fig:cycle}
    \end{figure}
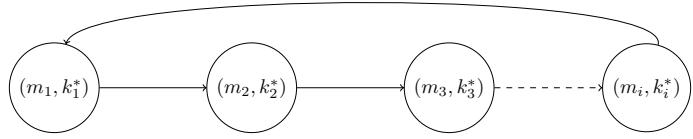

    We can then construct a new matching by reassigning each agent $m_j$ to the arm indicated by the node it points to, i.e., $k^*_{(j\%i) + 1}$, thus creating a new matching $G'$ with a higher total weight. We then show that if no agent is matched with their most preferred arm, there must be a cycle in $PG_{G^*}$, which contradicts the optimality of $G^*$ as established.

    If no agent is matched with their most preferred arm, consider an arbitrary agent $m$. We observe that the node $( m, k )$ has at least one outgoing edge. Suppose the node $( m', k' )$ is the any of them $( m, k )$ is pointing to. We also know that $( m', k' )$ must have at least one outgoing edge that does not point back to $( m, k )$; otherwise, $PG_{G^*}$ would contain a 2-cycle, contradicting the optimality of $G^*$. Therefore, $( m', k' )$ must point to another node, $( m'', k'' )$. To avoid forming a cycle, this chain of nodes would need to grow infinitely, which contradicts the assumption that $\mathcal{M}$ is finite. Thus, there must be at least one agent, denoted $m_1$, who is matched with their most preferred arm under $G^*$.

    To complete the proof, we then remove $m_1$ from $\mathcal{M}$ and $k_{m_1}^{G^*}$ from $\mathcal{K}$. We apply the same argument to show that there exists an agent $m_2$ who is matched with their most preferred arm after removing $k_{m_1}^{G^*}$, which could have been preferred to $k_{m_2}^{G^*}$ by $m_2$. This implies that $m_2$'s match under $G^*$ is either their most or second most preferred arm. We continue this argument for the remaining agents by induction, thereby completing the proof.

\end{proof}

\section{Auxiliary Lemmas}

\begin{lemma}\label{lemma:scaling_kl}
    For a given $0<p\leq q\leq1 $ as mean of Bernoulli distribution and $n\in \mathbb{R}^+$, the following inequality holds
    \begin{align*}
        \mathrm{kl}\left(\frac{p}{n},\frac{q}{n}\right)\leq \frac{1}{n} \mathrm{kl}\left(p,q\right).
    \end{align*}
\end{lemma}
\begin{proof}
    We first replace both sides of the inequality with the definition of $\mathrm{kl}$ and write it as:
    \begin{align}\label{ineq:kl_n}
        \frac{p}{n} \log\left(\frac{p}{q}\right) + \left(1 - \frac{p}{n}\right) \log\left(\frac{1 - \frac{p}{n}}{1 - \frac{q}{n}}\right) \leq \frac{1}{n}\left(p \log\left(\frac{p}{q}\right) + (1 - p) \log\left(\frac{1 - p}{1 - q}\right)\right),
    \end{align}

    For \eqref{ineq:kl_n} to hold, it suffices to prove:
    \begin{align*}
        \left(1 - \frac{p}{n}\right) \log\left(\frac{1 - \frac{p}{n}}{1 - \frac{q}{n}}\right) \leq \frac{1}{n} (1 - p) \log\left(\frac{1 - p}{1 - q}\right),
    \end{align*}
    which is equivalent to:
    \begin{align}\label{ineq:kl_scaling_2}
        \left(\frac{1 - \frac{p}{n}}{1 - \frac{q}{n}}\right)^{n - p} \leq \left(\frac{1 - p}{1 - q}\right)^{1 - p}.
    \end{align}

    Now we can observe that \(n = 1\) makes both sides of inequality \eqref{ineq:kl_scaling_2} equal. Assuming that \(n \geq 1\), we prove that \(\left(\frac{1 - \frac{p}{n}}{1 - \frac{q}{n}}\right)^{n - p}\) is a decreasing function in \(n\). Then we finish the proof by showing that the derivative of \(\left(\frac{1 - \frac{p}{n}}{1 - \frac{q}{n}}\right)^{n - p}\) with respect to \(n\) is always negative.

    After renaming \(v = n - p\) and \(u = \frac{n - p}{n - q}\), we can write:
    \[
        \frac{d u^v}{dn} = u^v \left(\frac{dv}{dn} \log u + \frac{v \frac{du}{dn}}{u}\right),
    \]

    To wrap up the proof, we need to show that this derivative is always negative, which implies \(u^v\) is decreasing in \(n\). For this argument to be correct, the following inequalities should hold:
    \begin{align}
        u^v \left(\frac{dv}{dn} \log u + \frac{v \frac{du}{dn}}{u}\right) \leq 0 & \implies \frac{dv}{dn} \log u + \frac{v \frac{du}{dn}}{u} \leq 0, \nonumber                         \\
                                                                                 & \implies \log \left(\frac{n - p}{n - q}\right) \leq \frac{q - p}{n - q}, \nonumber                  \\
                                                                                 & \implies \left(1 + \frac{q - p}{n - q}\right) \leq e^{\frac{q - p}{n - q}}. \label{ineq:derivation}
    \end{align}

    By the inequality \(\forall x > 0, \left(1 + x\right)^{\frac{1}{x}} \leq e\), inequality \eqref{ineq:derivation} is always correct. Thus, we proved that \(\left(\frac{1 - \frac{p}{n}}{1 - \frac{q}{n}}\right)^{n - p}\) is decreasing in \(n\), which completes our proof.

\end{proof}

\begin{lemma}\label{lemma:c_whhmab_kl1}
    For a given $m_1,m_2 \in \mathcal{M}$,$k_1,k_2\in \mathcal{K}$, and $n,n_1,n_2\in \mathbb{N}^+$ such that $\hat{\mu}_{m_1,k_1}(t) \leq \hat{\mu}_{m_2,k_2}(t)$ and $n_1\leq n_2$, We define $d^n_{m,k}(t)$ as
    $$d^n_{m, k}(t)\coloneqq\sup \left\{q \geq 0: n\ \mathrm{kl}\left(\hat{\mu}_{m,k}(t), q\right) \leq \log t+4 \log \log t\right\},$$
    the following inequalities hold
    \begin{enumerate}
        \item $d^{n_1}_{m, k}(t) > d^{n_2}_{m, k}(t),$
        \item $d^{n}_{m_1, k_1}(t) \leq d^{n}_{m_2, k_2}(t).$
    \end{enumerate}

\end{lemma}
\begin{proof}
    For the first part, according to the definition, \(d^n_{m,k}\) is decreasing in \(n\). This is because increasing \(n\) requires \(q\) to be closer to \(\hat{\mu}_{m,k}(t)\), as \(\mathrm{kl}(p', q')\) increases with \(q'\). Therefore, since \(n_1 \leq n_2\), we conclude that \(d^{n_1}_{m, k}(t) > d^{n_2}_{m, k}(t)\).

    For the second part we basically use the fact that $d^n_{m,k}(t)\geq \hat{\mu}_{m,k}(t)$. Assuming $q^*=d^n_{m_2,k_2}(t)$, then we know that
    \begin{align}\label{fct:lemma_kl_1}
        n\mathrm{kl}\left(\hat{\mu}_{m_2,k_2}(t),q^*\right) = \log t + \log\log t,
    \end{align}
    We then we can prove that $d^n_{m_1,k_1} \leq q^*$. Accordingly we put $q^*$ in definition of $d^n_{m_1,k_1}$ and write
    \begin{align*}
        \hat{\mu}_{m_1,k_1}(t) \leq \hat{\mu}_{m_2,k_2}(t) & \overset{}{\implies} q^*-\hat{\mu}_{m_2,k_2}(t) \leq q^*-\hat{\mu}_{m_1,k_1}(t),                                                    \\
                                                           & \overset{(a)}{\implies} \mathrm{kl}\left(\hat{\mu}_{m_2,k_2}(t),q^*\right) \leq \mathrm{kl}\left(\hat{\mu}_{m_1,k_1}(t),q^*\right), \\
                                                           & \overset{(b)}{\implies} n\mathrm{kl}\left(\hat{\mu}_{m_1,k_1}(t),q^*\right) \geq \log t + \log\log t,                               \\
                                                           & \implies d^n_{m_1,k_1}(t) \leq q^*,                                                                                                 \\
                                                           & \implies d^n_{m_1,k_1}(t) \leq d^n_{m_2,k_2}(t) .
    \end{align*}
    while (a) is implied by the fact that $\mathrm{kl}(p',q')$ increases as $\abs{q'-p'}$ grows and (b) is implied by equation \ref{fct:lemma_kl_1}.
\end{proof}

\end{document}